\documentclass[sn-mathphys-num]{sn-jnl}
\usepackage{multirow}
\usepackage{amsmath,amssymb,amsfonts}
\usepackage{amsthm}
\usepackage{mathrsfs}
\usepackage[title]{appendix}
\usepackage{xcolor}
\usepackage{textcomp}
\usepackage{enumitem}
\usepackage{manyfoot}
\usepackage{listings}%
\theoremstyle{plain}
\usepackage{algorithm}
\usepackage{algorithmic}
\usepackage{array}
\usepackage[algo2e]{algorithm2e}
\usepackage[caption=false,font=normalsize,labelfont=sf,textfont=sf]{subfig}
\usepackage{booktabs}
\usepackage{cite}
\usepackage{multirow}
\usepackage{graphicx} 
\usepackage{hyperref}
\usepackage{stfloats}
\usepackage{url}
\usepackage{verbatim}
\usepackage{hyperref}
\theoremstyle{thmstyleone}%
\newtheorem{theorem}{Theorem}
\newtheorem{proposition}[theorem]{Proposition}
\theoremstyle{thmstyletwo}%
\newtheorem{lemma}{Lemma}
\theoremstyle{thmstylethree}%
\usepackage[title]{appendix}
\usepackage{caption}
\usepackage{cite}
\usepackage{subcaption}

\begin{document}

\title{Hyperbolic Gaussian Blurring Mean Shift: A Statistical Mode-Seeking Framework for Clustering in Curved Spaces}

\author[1]{\fnm{Arghya} \sur{Pratihar}}\email{arghyapratihar24@gmail.com}

\author[2]{\fnm{Arnab} \sur{Seal}}\email{arnabseal37@gmail.com}
\author*[1]{\fnm{Swagatam} \sur{Das}}\email{swagatam.das@isical.ac.in}

\author[2]{\fnm{Inesh} \sur{Chattopadhyay}}\email{ineshchattopadhyay@gmail.com}

\affil*[1]{\orgdiv{Electronics and Communication Sciences Unit}, \orgname{Indian Statistical Institute}, \orgaddress{\city{Kolkata}, \country{India}}}

\affil[2]{\orgname{Indian Statistical Institute}, \orgaddress{\city{Kolkata}, \country{India}}}

\abstract{Clustering is a fundamental unsupervised learning task for uncovering patterns in data. While Gaussian Blurring Mean Shift (GBMS) has proven effective for identifying arbitrarily shaped clusters in Euclidean space, it struggles with datasets exhibiting hierarchical or tree-like structures. In this work, we introduce HypeGBMS, a novel extension of GBMS to hyperbolic space. Our method replaces Euclidean computations with hyperbolic distances and employs Möbius-weighted means to ensure that all updates remain consistent with the geometry of the space. HypeGBMS effectively captures latent hierarchies while retaining the density-seeking behavior of GBMS. We provide theoretical insights into convergence and computational complexity, along with empirical results that demonstrate improved clustering quality in hierarchical datasets. This work bridges classical mean-shift clustering and hyperbolic representation learning, offering a principled approach to density-based clustering in curved spaces. Extensive experimental evaluations on $11$ real-world datasets demonstrate that HypeGBMS significantly outperforms conventional mean-shift clustering methods in non-Euclidean settings, underscoring its robustness and effectiveness.}


\maketitle

\section{Introduction} \label{introduction}

Clustering is a fundamental unsupervised learning paradigm that organizes data into groups of similar entities, thereby revealing latent patterns and structural regularities within complex datasets. It enables the discovery of meaningful associations without requiring predefined class labels. Over the years, a broad spectrum of clustering algorithms has been developed, broadly categorized into \textit{hard} and \textit{soft} clustering methods. In hard clustering, each data point is assigned to a single cluster, whereas soft clustering allows points to belong to multiple clusters to varying degrees.  

Hard clustering techniques encompass several major families: \textit{center-based}, \textit{hierarchical}, \textit{distribution-based}, and \textit{density-based} methods. Center-based approaches—such as $k$-means \citep{kmeans}, spectral clustering \citep{spectral}, and kernel $k$-means \citep{kernel} measure similarity through proximity to cluster centroids. Hierarchical approaches, including hierarchical clustering \citep{hierarchical} and agglomerative clustering \citep{agglomerative}, assume that data points exhibiting closer pairwise similarities belong to the same cluster. Distribution-based methods, such as Expectation–Maximization (EM) for Gaussian Mixture Models \citep{gaussian} and its robust variants \citep{robustgmm}, model the data through underlying probabilistic distributions. Density-based approaches like DBSCAN \citep{dbscan}, HDBSCAN \citep{hdbscan}, and Mean Shift \citep{mean} instead infer cluster structure by identifying regions of high data density in the feature space.

Most of the aforementioned algorithms require the number of clusters to be specified \textit{a priori}. Among them, the \textbf{Mean Shift} algorithm has received significant attention because of its non-parametric, \textbf{mode-seeking} nature; it adaptively locates regions of high density without assuming a fixed number of clusters or a particular cluster shape. In particular, the \textbf{Gaussian Blurring Mean Shift (GBMS)} algorithm \citep{gbms} provides a smooth iterative framework for finding the modes of an underlying density in Euclidean space by progressively moving data points toward Gaussian-weighted means. This elegant formulation has found widespread success in computer vision, image segmentation, and manifold learning, owing to its simplicity, non-parametric flexibility, and theoretical grounding in kernel density estimation.  

Despite its strengths, traditional Euclidean mean-shift variants encounter difficulties when applied to datasets that exhibit \textit{hierarchical or tree-like structures}, such as taxonomies, linguistic ontologies, or social and biological networks. Such data are more naturally represented in \textbf{hyperbolic space}, where the geometry grows exponentially and inherently accommodates hierarchical expansion. Recent advances in hyperbolic representation learning have shown that embeddings in hyperbolic space can capture hierarchical relationships far more compactly than Euclidean embeddings, often requiring fewer dimensions while maintaining minimal distortion of the underlying structure.

To overcome these limitations, we introduce the \textbf{Hype}rbolic \textbf{G}aussian $\mathbf{B}$lurring \textbf{M}ean \textbf{S}hift (HypeGBMS) algorithm, which generalizes the GBMS procedure from Euclidean space to the \textbf{Poincaré ball model} of hyperbolic geometry. The key idea is to replace Euclidean distance computations with hyperbolic distances and to adapt the mean-shift update rule through \textbf{Möbius-weighted means}, thereby ensuring consistency with the non-Euclidean geometry of the underlying space. By adjusting the curvature parameter of the Poincaré ball, HypeGBMS provides a principled and flexible mechanism for \textbf{mode-seeking clustering} in hyperbolic space while preserving the desirable density-seeking behavior of the original GBMS. A major strength of the proposed framework lies in ensuring that all iterative updates remain confined to the Poincaré ball, thus maintaining the integrity of the hyperbolic manifold throughout the clustering process.

The Poincaré ball metric naturally preserves hierarchical relationships through its \textit{exponential distance scaling} and \textit{negative curvature}, which contrast sharply with the linear geometry of Euclidean space. As distances increase exponentially with displacement from the origin, hyperbolic space efficiently represents hierarchical data; levels of a hierarchy are automatically well-separated, enhancing cluster discernibility. Furthermore, the \textit{conformal property} of the Poincaré model preserves local angles and relative positioning, thereby retaining geometric fidelity when clustering complex, high-dimensional datasets. These properties make hyperbolic geometry especially well-suited for discovering latent hierarchies and non-Euclidean relationships where conventional Euclidean methods often fail. \\
\noindent
\textbf{Contributions.} The main contributions of this work are as follows:
\begin{enumerate}
    \item We propose \textbf{HypeGBMS}, a hyperbolic extension of Gaussian Blurring Mean Shift that embeds data within the Poincaré ball and ensures that all mode updates remain consistent with its geometric constraints.  
    \item We demonstrate that this formulation effectively captures hierarchical data structures and improves clustering performance in domains where Euclidean methods are inadequate.  
    \item We provide a detailed convergence and consistency analysis, along with a computational complexity study, and include supporting proofs and derivations in the Appendix.  
\end{enumerate}
In summary, this study bridges the conceptual divide between classical \textbf{mode-seeking} mean-shift clustering and modern \textbf{hyperbolic representation learning}, presenting a unified statistical framework for clustering in both flat and curved geometries.

\section{Related Works} \label{related}

Clustering methods have evolved across diverse algorithmic paradigms, with mean-shift and its extensions forming an important class of non-parametric density-based techniques. The \textit{mean-shift algorithm} estimates modes of an underlying probability density function by iteratively shifting data points toward regions of higher density. Carreira-Perpiñán \citep{carreira2015review} provided a comprehensive review of mean-shift algorithms, including Gaussian Blurring Mean-Shift (GBMS), hierarchical variants, and scale-space clustering approaches that extend its applicability across multiple domains. To improve computational efficiency in high-dimensional settings, Chakraborty et al. \citep{wbms} proposed feature-weighted GBMS for fast nonparametric clustering, demonstrating scalability without sacrificing clustering quality.

On a different note, hyperbolic geometry has gained traction for representing data with hierarchical or tree-like structures. Unlike Euclidean embeddings, which suffer from distortion when encoding exponentially expanding hierarchies, hyperbolic models provide a natural geometric prior. Earlier work, such as the Hierarchically Growing Hyperbolic SOM \citep{ontrup2006large}, explored hyperbolic topology for large-scale data exploration. More recently, hyperbolic graph embeddings have been applied to biological and neurological networks, where hierarchical structure is intrinsic \citep{baker2024hyperbolic}. Hyperbolic kernels and Poincaré embeddings have also been introduced to improve representation learning by aligning learning objectives with hyperbolic curvature \citep{fang2023poincare}.

Building upon this foundation, several methods have bridged \textit{clustering} with \textit{hyperbolic embeddings}. Chami et al. \citep{chami2020trees} developed hyperbolic hierarchical clustering techniques that recover tree-like structures directly from embeddings, enabling robust inference of latent hierarchies. Similarly, Zhao et al. \citep{zhao2020unsupervised} explored unsupervised embedding of hierarchical structures with variational autoencoders, highlighting how hyperbolic geometry enhances separability of nested clusters. Applications of hyperbolic contrastive learning further demonstrate the efficiency of clustering in negatively curved spaces \citep{wei2022hyperbolic}.

These advances motivate our proposed \textbf{HypeGBMS}, which integrates the density-seeking behavior of GBMS with the geometric priors of hyperbolic space. While GBMS provides smooth convergence to cluster modes in Euclidean settings, it fails to capture hierarchical structures. By embedding data into the Poincaré ball and redefining the update step through M\"obius weighted means, HypeGBMS ensures that cluster assignments remain consistent with the geometry of the Poincaré ball. In doing so, our approach unifies density-based clustering with modern hyperbolic representation learning, addressing the limitations of classical mean-shift algorithms in hierarchical domains.

\section{Preliminaries} \label{preliminaries}
We outline here the essential mathematical foundations of hyperbolic geometry relevant to this study. For a more comprehensive and rigorous exposition, we refer to the classical treatments in \citep{docarmo} and \citep{spivak}.

\subsection{Hyperbolic space} A Hyperbolic space, denoted by $\mathbb{H}^{n}$, is a non-Euclidean space of dimension $n$, defined as a simply connected Riemannian manifold with constant sectional curvature $-1$. The Killing–Hopf theorem \citep{lee2006riemannian} establishes that any two such spaces are isometric, implying that all realizations of $\mathbb{H}^{n}$ are equivalent up to isometry.

\begin{enumerate}

\item \textbf{Poincaré Disc Model:}  
The Poincaré disc provides a representation of hyperbolic space in which all points lie within the open unit ball in $\mathbb{R}^{n}$. In this model, geodesics correspond either to Euclidean diameters of the disc or to circular arcs orthogonal to the boundary. For any two points $\mathbf{X}$ and $\mathbf{Y}$ satisfying $\|\mathbf{X}\|, \|\mathbf{Y}\| < 1$, the hyperbolic distance is given by:

\begin{equation}
\label{eq:poincare_metric}
d(\mathbf{X},\mathbf{Y}) :=
\cosh^{-1}\left(1 + \frac{2\|\mathbf{X}-\mathbf{Y}\|^{2}}{(1 - \|\mathbf{X}\|^{2})(1 - \|\mathbf{Y}\|^{2})}\right).
\end{equation}

\item \textbf{Hyperboloid Model:}  
The Hyperboloid model is also referred to as the Minkowski or Lorentz model. This formulation represents hyperbolic space as the forward sheet $S^{+}$ of a two-sheeted hyperboloid embedded in $(n+1)$--dimensional Minkowski space. For points $\mathbf{X} = (x_{0}, x_{1}, \ldots, x_{n})$ and $\mathbf{Y} = (y_{0}, y_{1}, \ldots, y_{n})$ lying on $S^{+}$, the hyperbolic distance is defined as:

\begin{equation}
\label{eq:hyperboloid_metric}
d(\mathbf{X},\mathbf{Y}) := \cosh^{-1}\big(-M(\mathbf{X},\mathbf{Y})\big),
\end{equation}
where $M$ denotes the Minkowski inner product, defined by:
$$
    M \left(\left(x_{0}, x_{1}, \ldots, x_{n}\right),\left(y_{0}, y_{1}, \ldots, y_{n}\right)\right) 
$$
$$
    \qquad \qquad =-x_{0} y_{0}+\sum_{i=1}^{n} x_{i} y_{i} .
$$
\end{enumerate}
\subsection{Gyrovector Spaces}

The Gyrovector space framework offers a sophisticated algebraic structure suited for hyperbolic geometry, playing a role comparable to that of vector spaces in Euclidean geometry \citep{ungargyrovector}. We define
\[
\mathbb{D}_{c}^{p} := \left\{ \mathbf{v} \in \mathbb{R}^{p} \mid -c\|\mathbf{v}\|^{2} < 1 \right\},
\]
for curvature parameter \(c < 0\). In the special case \(c = 0\), this reduces to the Euclidean space \(\mathbb{R}^{p}\). When \(c < 0\), \(\mathbb{D}_{c}^{p}\) corresponds to an open ball with radius \(1/\sqrt{-c}\), and choosing \(c = -1\) recovers the unit ball \(\mathbb{D}^{p}\).

\begin{figure*}[h]
    \centering
    \includegraphics[width=0.85\linewidth]{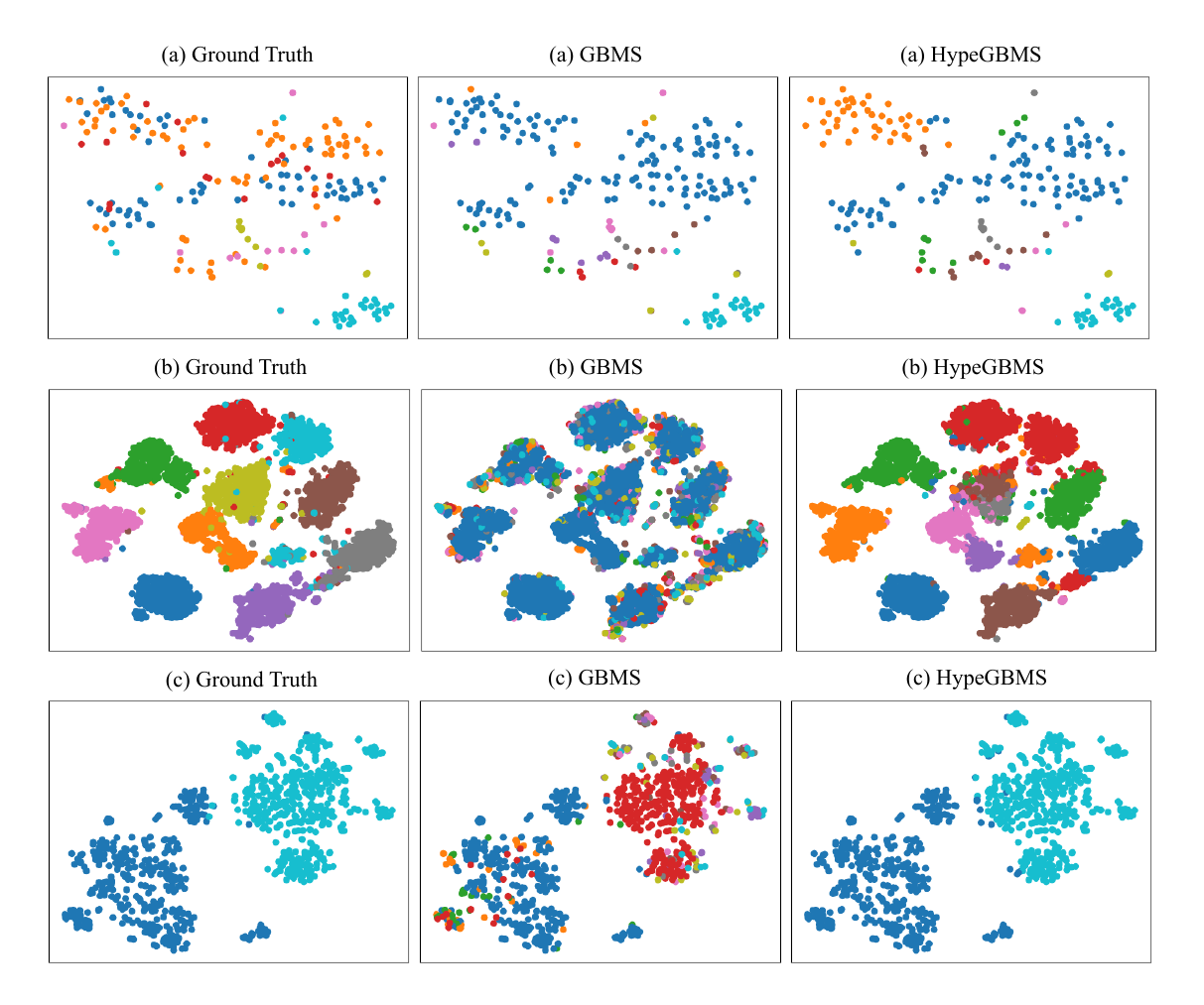}
    \caption{The t-SNE Visualisation of real datasets (a) Glass, (b) ORHD, (c) Phishing URL, respectively for GBMS and HypeGBMS (ours) clustering methods.}
    \label{fig:placeholder}
\end{figure*}

The Poincaré Disc Model is especially advantageous computationally, as most of its geometric operations admit closed-form expressions. In this setting, we briefly introduce Möbius gyrovector addition and Möbius scalar multiplication. Since isometries preserve the intrinsic geometric structure of hyperbolic spaces, the underlying additive and multiplicative algebraic frameworks remain invariant under model transitions and may therefore be transported isomorphically across distinct representations of hyperbolic geometry without altering their essential nature \citep{ratcliffehyperbolic}. To demonstrate this, we establish an explicit isometric correspondence between the Poincaré Disc Model and the hyperboloid model.

Let \(\mathbb{D}^n = \{ x \in \mathbb{R}^n : \|x\| < 1 \}\) denote the Poincaré Disc model. The corresponding hyperboloid model is given by
\[
\mathbb{H}^n = \{ x \in \mathbb{R}^{n+1} : -x_0^2 + x_1^2 + \cdots + x_n^2 = -1,\ x_0 > 0 \}.
\]
An isometric map between these two models is defined by
\[
\phi : \mathbb{D}^n \to \mathbb{H}^n,\qquad
x \mapsto \left( \frac{1 + \|x\|^2}{1 - \|x\|^2},\ \frac{2x}{1 - \|x\|^2} \right).
\]
This mapping preserves hyperbolic distance and induces the same Riemannian metric structure. Consequently, for the remainder of this work, we adopt the Poincaré Disc model as our primary representation \citep{hypenn,numericalstab}.\\

\noindent
\textbf{M\"obius addition.} The M\"obius addition of $\mathbf{v}$ and $\mathbf{w}$ in $\mathbb{D}_{c}^{p}$ is defined as :
{\small
\begin{equation}\label{eq: addition}
    \mathbf{v} \oplus_{c} \mathbf{w}:=\frac{\left(1-2 c\langle \mathbf{v}, \mathbf{w}\rangle-c\|\mathbf{w}\|^{2}\right) \mathbf{v}+\left(1+c\|\mathbf{v}\|^{2}\right) \mathbf{w}}{1-2 c\langle \mathbf{v}, \mathbf{w}\rangle+c^{2}\|\mathbf{v}\|^{2}\|\mathbf{w}\|^{2}}.
\end{equation}
}
In particular, when $c=0$, this conforms with the Euclidean addition of two vectors in $\mathbb{R}^{p}$.
However, it satisfies $\mathbf{v} \oplus_{c} \mathbf{0}=\mathbf{0} \oplus_{c} \mathbf{v}= \mathbf{v}$. Moreover, for any $\mathbf{v}, \mathbf{w} \in \mathbb{D}_{c}^{p}$, we have $(-\mathbf{v}) \oplus_{c} \mathbf{v}= \mathbf{v} \oplus_{c}(-\mathbf{v})=\mathbf{0}$ and $(-\mathbf{v}) \oplus_{c}\left(\mathbf{v} \oplus_{c} \mathbf{w}\right)= \mathbf{w}$.

\noindent
\textbf{M\"obius scalar multiplication.} For $c<0$, the M\"obius scalar multiplication of $\mathbf{v} \in \mathbb{D}_{c}^{p} \backslash\{\boldsymbol{0}\}$ by a real number $\lambda \in \mathbb{R}$ is defined as,
\begin{equation} \label{eq: mult}
    \lambda \otimes_{c} \mathbf{v}:=(1 / \sqrt{-c}) \tanh \left(\lambda \tanh ^{-1}(\sqrt{-c}\|\mathbf{v}\|)\right) \frac{\mathbf{v}}{\|\mathbf{v}\|}.
\end{equation}
and $\lambda \otimes_{c} \mathbf{0}:=\mathbf{0}$. As the parameter $c$ approaches zero, the expression reverts to conventional Euclidean scalar multiplication: $\lim_{c \rightarrow 0} \lambda \otimes_{c} \mathbf{v}=\lambda \mathbf{v}$. \\
The Hyperbolic Distance function on $\left(\mathbb{D}_{c}^{p}, g^{c}\right)$ is given by,
$$
    d_{\mathbb{H}}(\mathbf{v}, \mathbf{w})=(2 / \sqrt{-c}) \tanh ^{-1}\left(\sqrt{-c}\left\|-\mathbf{v} \oplus_{c} \mathbf{w}\right\|\right) .
$$
\textbf{Riemannian Log-Exp Map.} Given two points $x,y \in \mathbb{D}_{c}^{p}$, the logarithmic map at $x$ applied to $y$ is denoted as:
{\small
\begin{equation} \label{eq:logmap}
\log_{\mathbf{x}}(\mathbf{y}) = \frac{2}{\lambda_{\mathbf{x}}} \cdot \frac{\tanh^{-1}\left( \sqrt{-c} \cdot \| -\mathbf{x} \oplus_c \mathbf{y} \| \right)}{\| -\mathbf{x} \oplus_c \mathbf{y} \|} \cdot \left( -\mathbf{x} \oplus_c \mathbf{y} \right),
\end{equation}}
where ${\lambda_{\mathbf{x}}}= \frac{2}{1- \lambda{\|x\|}^2}$ and $\bigoplus_c$ represents the M\"obius gyrovector addition as Equation \eqref{eq: addition}.\\
Given $x\in \mathbb{D}_{c}^{p}$, $v \in T_x(\mathbb{D}_{c}^{p})$, the exponential map is defined as,

{\small
\begin{equation} \label{eq:expmap}
\exp_{\mathbf{x}}(\mathbf{v}) = \mathbf{x} \oplus_c \left( \tanh\left( \frac{\lambda_{\mathbf{x}} \cdot \sqrt{-c} \cdot \|\mathbf{v}\|}{2} \right) \cdot \frac{\mathbf{v}}{\sqrt{-c} \cdot \|\mathbf{v}\|} \right).
\end{equation}
As the parameter $c$ approaches $0$, 
\begin{equation} \label{eq:eucexp}
    \lim_{c \rightarrow 0} \exp_{\mathbf{x}}(\mathbf{v})=\mathbf{(x+v)},
\end{equation}}
\begin{equation} \label{eq:euclog}
\lim_{c \rightarrow 0} \log_{\mathbf{x}}(\mathbf{y})=\mathbf{(y-x)},
\end{equation}
which are the exponential and logarithmic maps in Euclidean space.

\section{HypeGBMS: our proposed method} \label{hypegbms}
We discuss our proposed method HypeGBMS in detail in Algorithm \ref{algo_hypegbms}. The updated steps in our algorithm are marked in blue.

\noindent
\begin{algorithm}[!ht]
\caption{\footnotesize Gaussian Blurring Mean-Shift (GBMS) Algorithm}
\label{alg:gbms}
\scriptsize
    \begin{algorithmic}[1]
        \STATE \textbf{Input:} Dataset $\mathbf{X} \in \mathbb{R}^{N \times p}$, bandwidth parameter $\sigma > 0$, convergence threshold $\tau > 0$, maximum iteration $T$.

        \vspace{0.5em}

        \STATE \textbf{Output:} Cluster assignments ${C_1, C_2, C_3, \cdots C_K.}$

        \vspace{0.5em}
        \let \oldnoalign \noalign
        \let \noalign \relax
        \let \noalign \oldnoalign
        \vspace{0.5em}
        \STATE \textbf{1. Initialize:} 
        \STATE \hspace{1em} Set iteration counter $t \gets 0$. 
        \STATE \hspace{1em} Initialise points $\mathbf{x}_i^{(0)} = \mathbf{x}_i$ for $i = 1, \dots, N$.
        \vspace{0.5em}
        \STATE \textbf{2. For} $t = 1$ to $T$:
        \STATE \hspace{1em} Compute pairwise squared Euclidean distances:
        \STATE \hspace{2em} $(d_{ij}^{(t)})^2 = \|\mathbf{x}_i^{(t)} - \mathbf{x}_j^{(t)}\|^2$
        \vspace{0.5em}
        \STATE \hspace{1em} Compute the Gaussian kernel weights:
        \STATE \hspace{2em} $w_{ij}^{(t)} \leftarrow \exp\left(-\dfrac{(d_{ij}^{(t)})^2}{2\sigma^2}\right)$
        
        \vspace{0.5em}
        \STATE \hspace{1em} Normalise the weights row-wise:
        \STATE \hspace{2em} $\bar{w}_{ij}^{(t)} = \dfrac{w_{ij}^{(t)}}{\sum_k w_{ik}^{(t)}}$
        \vspace{0.5em}
        \STATE \hspace{1em} Update each data point using weighted mean:
        \STATE \hspace{2em} $\mathbf{x}_i^{(t+1)} = \sum_{j=1}^N \bar{w}_{ij}^{(t)} \mathbf{x}_j^{(t)}$
        \vspace{0.5em}
        \STATE \hspace{1em} Compute average movement across all points:
        \STATE \hspace{2em} $\Delta^{(t)} = \dfrac{1}{N} \sum_{i=1}^N \|\mathbf{x}_i^{(t+1)} - \mathbf{x}_i^{(t)}\|$
        \vspace{0.5em}
        \STATE \hspace{1em} Check Convergence:
        \STATE \hspace{2em} If $\Delta^{(t)} < \tau$, \textbf{break}. \\
        \vspace{0.5em}
        \STATE \textbf{3. Cluster Assignment:}
        \STATE \hspace{1em} Construct adjacency matrix from final $\mathbf{X}^{(T)}$.
        \STATE \hspace{1em} Identify connected components as clusters.
    \end{algorithmic}
\end{algorithm}

\noindent \textbf{Step 1: Projecting onto the Poincaré Ball.}
We consider the dataset $X \in \mathbb{R}^{N \times p}$. We embed the data set into the hyperbolic space, here the Poincaré Ball model of radius $\frac{1}{\sqrt{-c}}, c <0 $ using the Riemannian Exponential map mentioned in Equation \eqref{eq:expmap}. We obtain $\tilde{\mathbf{X}}= \{\tilde{\mathbf{x}_1}, \tilde{\mathbf{x}_2}, \cdots, \tilde{\mathbf{x}_N}\} \in \mathbb{D}_{c}^{p}$, where $\mathbb{D}_{c}^{p}:=\left\{\mathbf{x} \in \mathbb{R}^{p} \mid -c\|\mathbf{x}\|^{2}<1\right\}$.

\begin{algorithm}[!ht]
\caption{\footnotesize Hyperbolic Gaussian Blurring Mean-Shift (HypeGBMS) Algorithm}
\label{algo_hypegbms}
\scriptsize
    \begin{algorithmic}[1]
        \STATE \textbf{Input:} Dataset $\mathbf{X} \in \mathbb{R}^{N \times p}$, bandwidth parameter $\sigma > 0$, convergence threshold $\varepsilon > 0$, minimum cluster separation $\delta > 0$, maximum iteration $T$, curvature parameter $c < 0$.
        \vspace{0.5em}
        \STATE \textbf{Output:} Cluster assignments ${C_1, C_2, C_3, \cdots C_K.}$
        \let \oldnoalign \noalign
        \let \noalign \relax
        \let \noalign \oldnoalign
        \vspace{0.5em}
         
        \STATE \textbf{1. Initialize:} 
        \STATE \hspace{1em} Set iteration counter $t \gets 0$.
        \textcolor{blue}{
        \STATE \hspace{1em} $\tilde{\mathbf{X}} \gets \text{project onto Poincaré ball}  (\mathbf{X}, c)$}
        \vspace{0.5em}
        \STATE \textbf{2. For} $t = 1$ to $T$:
        \STATE \hspace{1em} \textcolor{blue}{Compute pairwise squared hyperbolic distances:
        \STATE \hspace{2em} $(d_{ij}^{(t)})^{2} = d_{\mathbb{H}}^2(\tilde{\mathbf{x}}_i^{(t)}, \tilde{\mathbf{x}}_j^{(t)})$}
        \vspace{0.5em}
        \STATE \hspace{1em}Compute the Gaussian kernel weights
        \STATE \hspace{2em} $w_{ij}^{(t)} \leftarrow\exp\left(-\dfrac{(d_{ij}^{(t)})^{2}}{2\sigma^2}\right)$
        \vspace{0.5em}
        \STATE \hspace{1em}Normalise the weights row-wise:
        \STATE \hspace{2em} $\bar{w_{ij}}^{(t)} = \dfrac{{w_{ij}}^{(t)}} {\sum_k {w_{ik}}^{(t)}}$ \\
        \vspace{0.5em}\textcolor{blue}{
        \STATE \hspace{1em}Compute M\"obius Weighted Mean: \\
        \STATE \hspace{2em} $\tilde{\mathbf{x}_i}^{(t+1)}=\bigoplus_{j=1}^N \bar{w}_{ij}^{(t)}\otimes_c\tilde{\mathbf{x}_j}^{(t)}$} \\ 
        \vspace{0.5em}\textcolor{blue}{
        \STATE \hspace{1em}Compute average movement:
        \STATE \hspace{2em} $\Delta^{(t)} = \frac{1}{N} \sum_{i=1}^N \left\| \log_{\tilde{\mathbf{x}}_i^{(t)}} \left( \tilde{\mathbf{x}}_i^{(t+1)} \right) \right\|$}
        \vspace{0.5em}
        \STATE \hspace{1em} \textcolor{blue}{Check Convergence
        \STATE \hspace{2em} If $\Delta^{(t)} < \varepsilon$},
        \textbf{break}.\\
        \vspace{0.5em}
        \STATE \textbf{3. Cluster Assignment:} \textcolor{blue}{
        \STATE \hspace{1em} Construct adjacency matrix from final $\tilde{\mathbf{X}}^{(T)}$.}
        \STATE \hspace{1em} Identify connected components as clusters.
    \end{algorithmic}
\end{algorithm}

\noindent
\textbf{Step 2: Iterative Update Loop (for $t = 1$ to $T$)}

\noindent
\textbf{Computing Pairwise Squared Hyperbolic Distances:}
We compute the pairwise squared hyperbolic distances as, 
\begin{equation} \label{eq:distance}
    (d_{ij}^{(t)})^{2} = d_{\mathbb{H}}^2(\tilde{\mathbf{x}}_i^{(t)}, \tilde{\mathbf{x}}_j^{(t)})
\end{equation}
where the poincar\'e ball distance is given by:
\begin{equation} \label{eq:hyperbolic_dist}
d_{\mathbb{H}}(\mathbf{x}, \mathbf{y}) = \frac{2}{\sqrt{-c}} \tanh^{-1}\left( \sqrt{-c} \| -\mathbf{x} \oplus_c \mathbf{y} \| \right)
\end{equation}

\noindent
\textbf{Computing Gaussian Kernel Weights:} We construct the weight matrix, $\textbf{W}= (w_{ij})_{N\times N}$ where the entries are defined as,
\begin{equation} \label{eq:weight}
    w_{ij}^{(t)} = \exp\left(-\frac{(d_{ij}^{(t)})^{2}}{2\sigma^2}\right),
\end{equation}
where, $d_{ij}$ are defined as Equation \eqref{eq:distance}, $\sigma$ is the bandwidth parameter.

\noindent
\textbf{Normalize the weights Row-wise:} Now, we normalize the weight matrix, $W$ as
\begin{equation} \label{eq:normalise}
    \bar{w}_{ij}^{(t)} = \frac{{w_{ij}}^{(t)}}{\sum_{k=1}^n {w_{ik}}^{(t)}}
\end{equation}

\noindent
\textbf{Computing Möbius Weighted Mean for each point:} We define the m\"obius weighted mean as, $\bigoplus_{j=1}^N \bar{w}_{ij}^{(t)}\otimes_c\tilde{\mathbf{x}_j}^{(t)}$, where $\bigoplus_j$ is M\"obius addition defined as \ref{eq: addition} and $\otimes_c$ is the M\"obius scalar multiplication as \ref{eq: mult}. 

\begin{equation} \label{eq:mob_mean}
\tilde{\mathbf{x}}_i^{(t+1)} \leftarrow \text{M\"obiusWeightedMean}(\{ \tilde{\mathbf{x}}_j^{(t)} \}, \{ {\bar{w}_{ij}}^{(t)} \}, c)
\end{equation}


\noindent
\textbf{Computing Average Movement:} We define the average movement as
\begin{equation} \label{eq:logstop}
\Delta^{(t)} = \frac{1}{N} \sum_{i=1}^N \left\| \log_{\tilde{\mathbf{x}}_i^{(t)}} \left( \tilde{\mathbf{x}}_i^{(t+1)} \right) \right\|,
\end{equation}
where the logarithm map is defined as \ref{eq:logmap}.

\noindent
\textbf{Checking Convergence:}\\
If $\Delta^{(t)} < \varepsilon$ or $t = T$, then we stop the iteration.

\noindent
\textbf{Step 3: Cluster Assignment.}
Let $\tilde{\mathbf{X}}^{(T)}=\{\tilde{\mathbf{x}}_1^{(T)},\dots,\tilde{\mathbf{x}}_N^{(T)}\}$. We construct an adjacency matrix $A \in \{0,1\}^{N \times N}$ defined by
\[
A_{ij} =
\begin{cases}
1, & \text{if } d_{\mathbb{H}}\left(\tilde{\mathbf{x}}_i^{(T)}, \tilde{\mathbf{x}}_j^{(T)}\right) \le \delta, \\[6pt]
0, & \text{otherwise},
\end{cases}
\]
where $d_{\mathbb{H}}(\cdot,\cdot)$ is the hyperbolic distance and $\delta > 0$ is a  threshold defined as minimum cluster separation. Interpreting $A$ as the adjacency matrix of an undirected graph $G=(V,E)$ with vertex set $V=\{1,\dots,N\}$ and edges $(i,j)\in E$ whenever $A_{ij}=1$, the final clustering is obtained by extracting its connected components. Thus, the resulting cluster set is $\mathbf{C}=\{C_1,\dots, C_K\}$, where each $C_k$ corresponds to one connected component and $K$ denotes the number of such components.

\section{Stopping Criteria for HypeGBMS} \label{stopping criteria}
The HypeGBMS algorithm terminates when the average movement drops below a certain tolerance limit, $\varepsilon = 10^{-5}$, that is:
$$
  \Delta^{(t)}< \varepsilon,
$$ 
where $\Delta^{(t)}$ is defined as Equation \eqref{eq:logstop}.
In iteration $t$, the update $\{\Delta_i^{(t)}= \left\|\log_{\tilde{{x}}_i^{(t)}}\tilde{{x}}_i^{(t+1)}\right\|\}$ for $i= 1,2, \cdots N$ takes at most $K$ different values (for $K$ clusters) and a histogram of $\{\Delta_i^{(t)}\}_{i=1}^N$ has $K$ or fewer nonempty bins. 
\noindent
In summary, our stopping criterion is:
\begin{equation}
  {\Delta^{(t)}} < \varepsilon
  \quad \text{OR} \quad
  |H^{(t+1)} - H^{(t)}| < \gamma,
\end{equation}
where the Shannon entropy is given by:
\[
  H^{(t)} = -\sum_{k=1}^S p_k \log p_k,
\]
where $p_k$ is the normalized bin counts such that $\sum_{k=1}^S p_k = 1$. The number of bins $S$ should be larger than the number of clusters but smaller than $N$; we use $S = 0.9N$. Checking this criterion requires the time complexity of $\mathcal{O}(N)$. 

\section{Computational Complexity Analysis of HypeGBMS} \label{complexity}

In this section, we will present the step-by-step complexity analysis of our proposed method, HypeGBMS.

\noindent
\textbf{Step 1: Projecting onto the Poincaré Ball.} We consider the dataset $X \in \mathbb{R}^{N \times p}$ and project onto the Poincar\'e Ball. This step requires the time complexity of 
$\mathcal{O}(N \times p)$.

\noindent
\textbf{Step 2: Iteration Loop ($t = 1$ to $T$)}

\noindent
\textbf{Computing Pairwise Squared Hyperbolic Distances.}
We compute the hyperbolic distance between every pair of points by Equation \eqref{eq:distance}. Each hyperbolic distance computation involves dot products in the $p$-dimensional space that requires the time complexity of $\mathcal{O}(p)$. So, for a total of $N^2$ pairs, the time complexity is $\mathcal{\mathcal{O}}(N^2 \times p)$.

\noindent
\textbf{Computing Gaussian Kernel Weights.}
We apply the exponential function to $N^2$ distance values. The time complexity for this step is
$\mathcal{\mathcal{O}}(N^2)$.
For each of $N$ rows: sum $N$ weights, then dividing $N$ weights to normalise requires the time complexity of $\mathcal{O}(N^2)$.

\noindent
\textbf{Computing Möbius Weighted Mean for each point.}
Each point is updated using the weighted mean of $N$ points. Each Möbius weighted mean computation in the Poincaré disc model involves Möbius addition and Möbius scalar multiplication, which require the time complexity of $\mathcal{\mathcal{O}}(N \times p)$ for each point. The total complexity for this step becomes $\mathcal{\mathcal{O}}(N^2 \times p)$.

\begin{figure*}
    \centering
    \includegraphics[width= \linewidth]{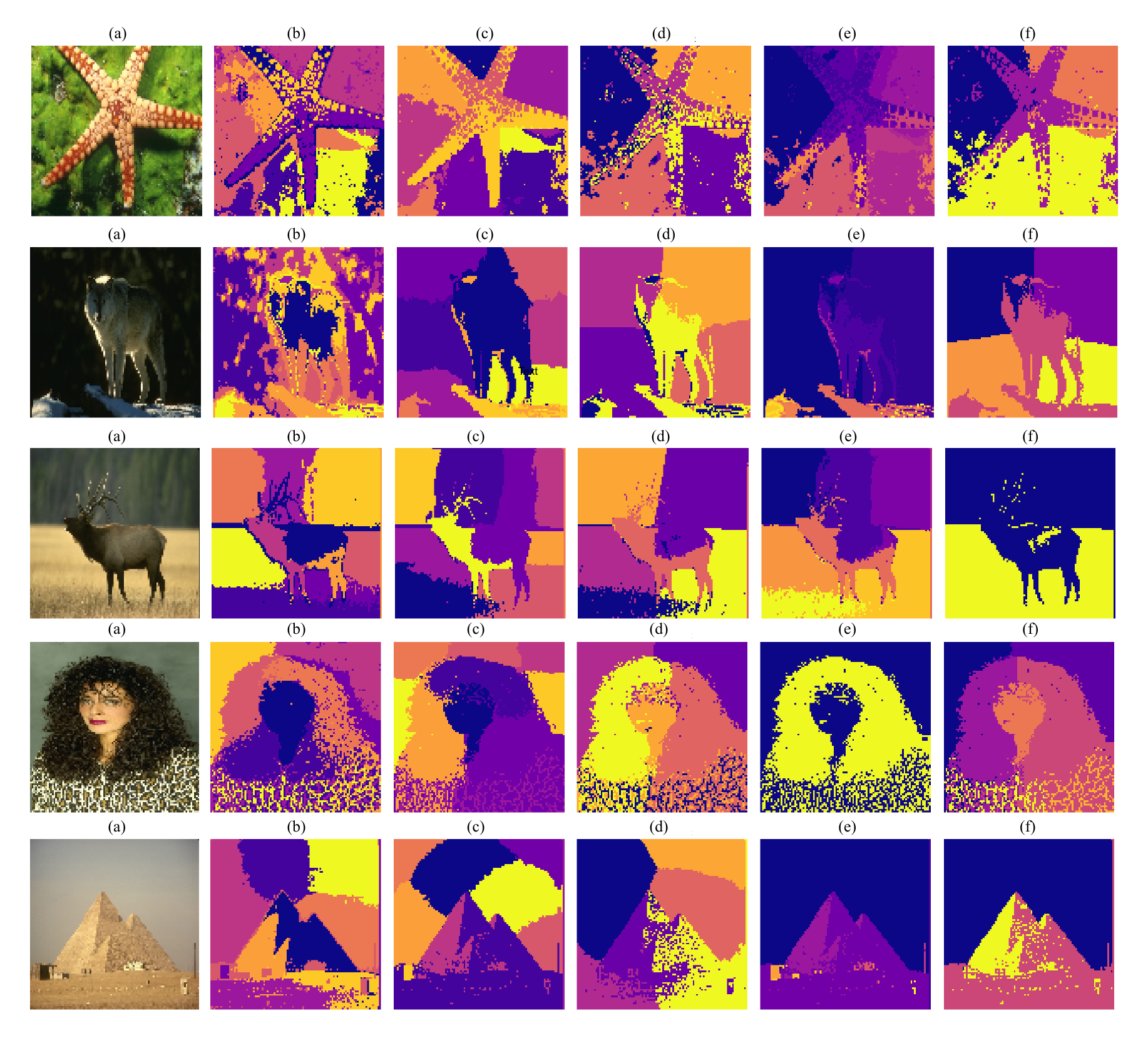}
    \caption{Qualitative results on the BSDS500 dataset (a) Image, (b) GMS, (c) Gridshift, (d) QuickMeanshift, (e) GBMS, (f) HypeGBMS(ours).}
    \label{fig:bsds500}
\end{figure*}

\noindent
\textbf{Computing the Average Movement.}
We compute the logarithmic map for $N$ points, which requires $N \times \mathcal{\mathcal{O}}(p)$. So, the time complexity for this step in each iteration becomes $\mathcal{O}(N \times p)$. \\
\noindent
\textbf{Step 3: Cluster Assignment.} After convergence, from the final $\tilde{\mathbf{X}}^{(T)}$, computing pairwise distances again and thresholding yields $\mathcal{\mathcal{O}}(N^2 \times p)$.
Interpreting the adjacency matrix as a dense graph with $N$ nodes and $\mathcal{O}(N^2)$ edges costs $\mathcal{O}(N^2)$. Thus, the cost in Step~3 is $\mathcal{\mathcal{O}}(N^2 \times p)$. \\
Therefore, the overall time complexity for our proposed algorithm is $\mathcal{\mathcal{O}}(T \times N^2 \times p).$ \\

\noindent
\textbf{Remark on Scalable Approximations.}
Although the worst-case computational cost of HypeGBMS is 
$\mathcal{\mathcal{O}}(T N^2 p)$ due to the pairwise hyperbolic distance evaluations, this quadratic scaling becomes prohibitive for large-scale datasets. Several well-established approximation techniques from large-scale kernel methods and nearest-neighbor search can be introduced to reduce the effective cost. For example, tree-based search structures such as KD-trees and ball trees \citep{friedman1977kd,omohundro1989balltree} 
and their modern descendants (e.g., FLANN \citep{muja2009flann}) provide sublinear approximate nearest-neighbor queries in general metric spaces, including negatively curved ones. Similarly, kernel 
approximation schemes such as the Nystr{\"o}m method 
\citep{williams2001nystrom} and random feature expansions 
\citep{rahimi2007randomfeatures} reduce the effective kernel complexity from $\mathcal{O}(N^{2})$ to near-linear in $N$. These ideas parallel acceleration strategies used in fast density-based clustering methods such as QuickShift++ \citep{jiang2018quickshiftpp}. Adapting 
such approximation frameworks to the hyperbolic setting offers a promising direction for developing scalable variants of HypeGBMS.

\section{Convergence Analysis of HypeGBMS} \label{convergence}

This section provides a rigorous convergence analysis of the proposed HypeGBMS algorithm on the Poincaré ball model $(\mathbb{D}_c^p,g_c)$ of constant curvature $c<0$. 
Our analysis explicitly distinguishes the Möbius weighted mean used in Algorithm \ref{algo_hypegbms} from the Riemannian Fréchet mean required for standard convergence proofs and establishes curvature–radius dependent error bounds ensuring asymptotic convergence.

Let $(\mathbb{D}_c^p,g_c)$ be the $p$-dimensional Poincaré ball with sectional curvature $c=-\kappa<0$ and hyperbolic distance $d_{\mathbb{H}}$. The Möbius weighted mean ($\bar{x}_M$) used in Algorithm \ref{algo_hypegbms} is defined as Equation \eqref{eq:mob_mean}. The weighted Riemannian Fréchet mean is defined as: 
\begin{equation}
\bar{x}_F = \arg\min_{z\in\mathbb{D}_c^p}
F(z) := \arg\min_{z\in\mathbb{D}_c^p} \sum_{j=1}^N w_{ij}\, d_{\mathbb{H}}^2(z,x_j).
\end{equation}
\begin{proposition}
\label{prop:mobius-frechet}
Let $\{x_j\}_{j=1}^N\subset \mathbb{D}_c^p$ lie within a geodesic ball $B_\mathbb{H}(x_0,r)$ of radius $r<1$ and curvature magnitude $\kappa=-c>0$. 
Then there exists a constant $C_p>0$, depending only on dimension $p$, such that
\begin{equation}
\label{eq:mobius_error}
d_{\mathbb{H}}(\bar{x}_M,\bar{x}_F) \le C_p\,\kappa\, r^3.
\end{equation}
Moreover, $\bar{x}_M,\bar{x}_F\in B_\mathbb{H}(x_0,r)$.
\end{proposition}
\noindent
The proof is shown in the Appendix. \\

\noindent
\textbf{Remark.} This proposition suggests that when all points lie in a sufficiently small ball,
the Möbius weighted mean ($\bar{x}_M$) is a first-order (in the radius $r$) approximation [$\because r = \frac{1}{\sqrt{\kappa}}$] of the Fréchet mean ($\bar{x}_F$). As in $r \to 0$, the two notions of mean converge, and both reduce to the Euclidean barycenter in the tangent space.

\subsection{Boundedness of Iterates}

\begin{lemma}[Iterates remain in a geodesically convex subset]
\label{lem:bounded}
Suppose the initial points $\{\tilde{x}_i^{(0)}\}$ lie in a compact, geodesically convex subset $\mathcal{U}\subset\mathbb{D}_c^p$ of radius $r<1$. 
Then for all $t\ge0$,
\[
\tilde{x}_i^{(t)}\in \mathcal{U},\qquad
d_{\mathbb{H}}(\tilde{x}_i^{(t+1)},\tilde{x}_k^{(t+1)}) 
\le (1+\alpha)\, d_{\mathbb{H}}(\tilde{x}_i^{(t)},\tilde{x}_k^{(t)}),
\]
where $\alpha = C_p \kappa r^2$.
\end{lemma} 
\begin{proof}
Let $T_i^{(t)}:\mathbb{D}_c^p\to\mathbb{D}_c^p$ be the update map such that
$$
    T_i^{(t)}(x)=\text{MöbiusWeightedMean}(\{\tilde{x_j}^{(t)},\{\bar{w}_{ij}^{(t)}\},c).
$$
Let $\Phi_i^{(t)}:\mathbb{D}_c^p\to\mathbb{D}_c^p$ be the update map such that
$$
 \Phi_i^{(t)}(x) = \arg\min_{x\in\mathbb{D}_c^p} \sum_{j=1}^N w_{ij}\, d_{\mathbb{H}}^2(x,\tilde{x_j}^{(t)})  
$$
From Proposition \ref{prop:mobius-frechet}, $T_i^{(t)}$ is an $\mathcal{O}(\kappa r^3)$-approximation of the exact Fréchet mean map $\Phi_i^{(t)}$. Weighted Riemannian barycentric maps are non-expansive in any $\mathrm{CAT}(0)$ space (including the Poincaré ball), thus:
$d_{\mathbb{H}}(\Phi_i^{(t)}(x),\Phi_k^{(t)}(y))\le d_{\mathbb{H}}(x,y)$ \citep{afsari2011riemannian}. 
By the triangle inequality,
\begin{align*}
d_{\mathbb{H}}(T_i^{(t)}(x),T_k^{(t)}(y))
& \le 
d_{\mathbb{H}}(T_i^{(t)}(x),\Phi_i^{(t)}(x)) + d_{\mathbb{H}}(\Phi_i^{(t)}(x),\Phi_k^{(t)}(y)) + d_{\mathbb{H}}(T_k^{(t)}(y),\Phi_k^{(t)}(y)) \\
& \le (1+ \mathcal{O}(\kappa r^2)) d_{\mathbb{H}}(x,y), \end{align*} giving the required $(1+\alpha)$ factor, where $\alpha = C_p \kappa r^2$. Compactness of $\mathcal{U}$ ensures iterates remain in $\mathcal{U}$.
\end{proof}

\subsection{Convergence of HypeGBMS}

\begin{theorem}[Convergence of HypeGBMS]
\label{thm:convergence}
Let $\{\tilde{x}_i^{(t)}\}_{t\ge0}$ denote the sequence generated by Algorithm \ref{algo_hypegbms}. 
Assume all iterates lie in a convex ball $\mathcal{U}\subset\mathbb{D}_c^p$ of radius $r<1$ and curvature magnitude $\kappa=-c$. 
Then:
\begin{enumerate}
    \item The sequence $\hat{f}(\tilde{x}_i^{(t)})$ with
    \[
    \hat{f}(x)=\frac{1}{N}\sum_{j=1}^N \exp\!\left(-\frac{d_{\mathbb{H}}^2(x,\tilde{x}_j)}{2\sigma^2}\right)
    \]
    is monotone non-decreasing up to additive errors $\mathcal{O}(\kappa r^3)$ and convergent.
    \item Each $\tilde{x}_i^{(t)}$ converges to a point $\tilde{x}_i^\ast$ satisfying 
    $\|\nabla_\mathbb{H} \hat{f}(\tilde{x}_i^\ast)\|\le C_p \kappa r^3$.
\end{enumerate}
In the limit $\kappa r^3\to0$, $\tilde{x}_i^\ast$ is an exact stationary point of $\hat{f}$.
\end{theorem}
\begin{proof}
By Lemma \ref{lem:bounded}, iterates remain in a compact convex set $\mathcal{U}$ where curvature, metric, and derivatives of $\hat{f}$ are bounded.  
Let $\Phi_i^{(t)}$ be the exact Fréchet mean update. 
Then $\Phi_i^{(t)}$ performs a Riemannian gradient ascent on $\hat{f}$ with step size $\eta=1/\sigma^2$ \citep{afsari2011riemannian}, implying
$\hat{f}(\Phi_i^{(t)}(\tilde{x}_i^{(t)}))-\hat{f}(\tilde{x}_i^{(t)})\ge c_1\|\nabla_\mathbb{H}\hat{f}(\tilde{x}_i^{(t)})\|^2$ for some $c_1>0$, where the Riemannian gradient of $\hat{f}$ is
\begin{equation}
    \nabla_\mathbb{H} \hat{f}(\mathbf{x})
    = \frac{1}{N\sigma^2} \sum_{j=1}^N K_\sigma(\mathbf{x}, \tilde{\mathbf{x}}_j) \, \log_{\mathbf{x}}(\tilde{\mathbf{x}}_j),
\end{equation}
where $K_\sigma(\mathbf{x}, \tilde{\mathbf{x}}_j) = \exp\!\big(-d_\mathbb{H}^2(\mathbf{x}, \tilde{\mathbf{x}}_j)/(2\sigma^2)\big)$ and $\log_{\mathbf{x}}(\cdot)$ is the Riemannian logarithmic map as \ref{eq:logmap}. 
Using Proposition \ref{prop:mobius-frechet}, the Möbius update $T_i^{(t)}$ satisfies
$d_{\mathbb{H}}(T_i^{(t)}(\tilde{x}_i^{(t)}),\Phi_i^{(t)}(\tilde{x}_i^{(t)}))\le C_p\kappa r^3$, 
so by Lipschitz continuity of $\hat{f}$ we get
\begin{equation} \label{ineq:mob-mean}
    \hat{f}(T_i^{(t)}(\tilde{x}_i^{(t)}))
\ge \hat{f}(\Phi_i^{(t)}(\tilde{x}_i^{(t)})) - C_p\kappa r^3
\ge 
\hat{f}(\tilde{x}_i^{(t)}) + c_1\|\nabla_\mathbb{H}\hat{f}(\tilde{x}_i^{(t)})\|^2 - C_p\kappa r^3.
\end{equation}
Hence $\{\hat{f}(\tilde{\mathbf{x}}_i^{(t)})\}$ is a monotone non-decreasing sequence bounded above (since $0 < \hat{f}(\mathbf{x}) \leq 1$). Hence, it converges by the Monotone convergence theorem \citep{rudin}. 
Summing the inequality (\ref{ineq:mob-mean}) over $t$ yields
$\sum_t\|\nabla_\mathbb{H}\hat{f}(\tilde{x}_i^{(t)})\|^2<\infty$, implying $\|\nabla_\mathbb{H}\hat{f}(\tilde{x}_i^{(t)})\|\to0$ up to $\mathcal{O}(\kappa r^3)$, so each limit point $\tilde{x}_i^\ast$ satisfies the claimed condition.
\end{proof}
The above results establish that the Möbius update constitutes an \emph{inexact Riemannian gradient ascent} whose inexactness vanishes with both curvature magnitude and local cluster radius. 
Therefore, HypeGBMS inherits the convergence behavior of classical GBMS in the small-curvature regime and converges to approximate Fréchet stationary points under bounded curvature.

\section{Statistical Consistency of HypeGBMS} \label{consistency}

We now establish a consistency result for the cluster centroids obtained by the proposed HypeGBMS method. Let $(\mathbb{D}_c^p, g_c)$ denote the Poincar\'e ball model of hyperbolic space with curvature $c<0$ and metric tensor $g$. Let $d_{\mathbb{H}}(\cdot,\cdot)$ be the corresponding hyperbolic distance as mentioned in Equation \eqref{eq:hyperbolic_dist}. Let $X_1,\dots,X_N \overset{iid}{\sim} P$ with density $f$ supported on a compact subset of $\mathbb{D}_c^p$. 
\noindent
For bandwidth $\sigma>0$, the kernel is
\begin{equation}
  K_\sigma(x,y) = \exp\!\Big(-\tfrac{d_{\mathbb{H}}^2(x,y)}{2\sigma^2}\Big),  
\end{equation}
and the empirical kernel density estimator (KDE) is :
\[
\hat f_\sigma(x) = \frac{1}{N}\sum_{i=1}^N K_\sigma(x,X_i).
\]
The population smoothed density is
\[
f_\sigma(x) = \int_{\mathbb{D}_c^p} K_\sigma(x,y) f(y)\, d\mathrm{vol}_{\mathbb{H}}(y).
\]

\noindent 
Modes of $f_\sigma$ are :
\[
\mathcal{M}_\sigma = \{m_1,\dots,m_K\}, \quad m_j = \arg\max_{x\in U_j} f_\sigma(x),
\]
where $U_j$ are disjoint neighborhoods of each mode. The empirical analogues $\hat{\mathcal{M}}_\sigma$ are the stationary points obtained by HypeGBMS.

\begin{figure*}[h]
    \centering
    \includegraphics[width= \linewidth]{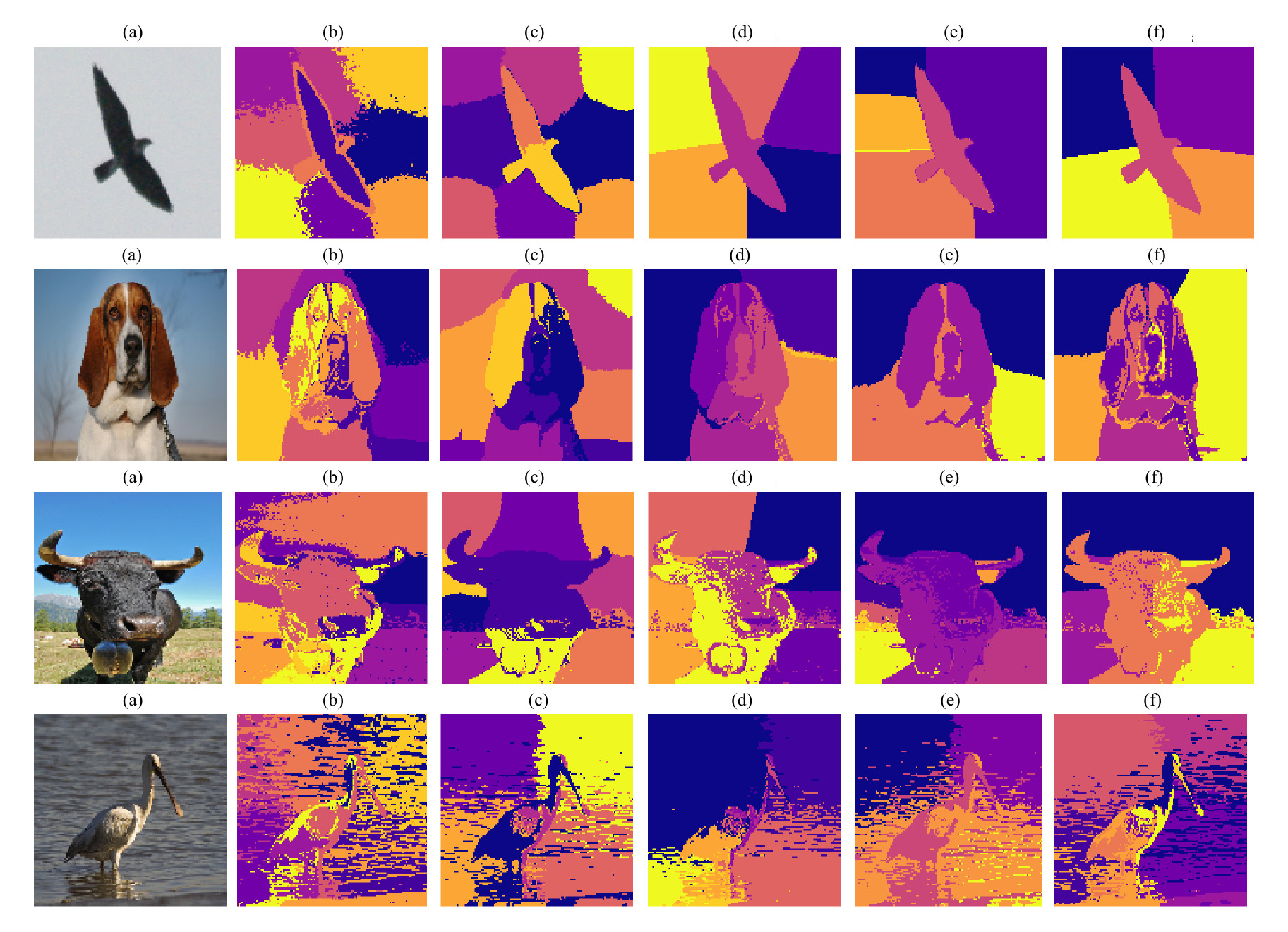}
    \caption{Qualitative results on the PASCAL VOC 2012 dataset (a) Image, (b) GMS, (c) Gridshift, (d) QuickMeanshift, (e) GBMS, (f) HypeGBMS(ours).}
    \label{fig:pascal-voc}
\end{figure*}
\begin{lemma} \label{lem:uniform}
Suppose $f$ is continuous and bounded, and the kernel $K_\sigma$ is smooth in $x$. Then as $N \to \infty$,
\[
\sup_{x \in \mathbb{D}_c^{p}} \big|\hat f_\sigma(x) - f_\sigma(x)\big| \xrightarrow{a.s.} 0,
\]
and
\[
\sup_{x \in \mathbb{D}_c^{p}} \|\nabla_{\mathbb{H}} \hat f_\sigma(x) - \nabla_{\mathbb{H}} f_\sigma(x)\| \xrightarrow{a.s.} 0.
\]
\end{lemma}

\begin{proof}
Since $f$ is supported on a compact subset of $\mathbb{D}_c^p$, there exists a
compact geodesically convex set $\mathcal{U}\subset\mathbb{D}_c^p$ such that
$\mathrm{supp}(f)\subset\mathcal{U}$. It suffices to prove the uniform
convergence on $\mathcal{U}$:
\begin{align*}
  & \sup_{x\in\mathcal{U}} \big|\hat f_\sigma(x) - f_\sigma(x)\big| \xrightarrow{a.s.} 0,
\\
& \sup_{x\in\mathcal{U}} \|\nabla_{\mathbb{H}} \hat f_\sigma(x) - \nabla_{\mathbb{H}} f_\sigma(x)\| \xrightarrow{a.s.} 0,  
\end{align*}
since outside $\mathcal{U}$ the density $f$ vanishes and both $f_\sigma$ and
$\hat f_\sigma$ are determined by their values on $\mathcal{U}$. For each fixed $x\in\mathcal{U}$, define the function
\[
\varphi_x(y) := K_\sigma(x,y)
= \exp\!\Big(-\tfrac{d_{\mathbb{H}}^2(x,y)}{2\sigma^2}\Big),
\qquad y\in\mathbb{D}_c^p.
\]
Then
\[
\hat f_\sigma(x) = \frac{1}{N} \sum_{i=1}^N \varphi_x(X_i),
\qquad
f_\sigma(x) = \mathbb{E}[\varphi_x(X_1)].
\]
Because $K_\sigma$ is continuous and $d_{\mathbb{H}}(x,y)\ge 0$, there exists
a finite constant $M>0$ such that $|\varphi_x(y)|\le M$ for all
$x\in\mathcal{U}$ and all $y\in\mathbb{D}_c^p$. For each fixed $x$, the
random variables $\{\varphi_x(X_i)\}_{i\ge 1}$ are i.i.d.\ with
$\mathbb{E}|\varphi_x(X_1)|<\infty$, so by the strong law of large numbers,
\[
\frac{1}{N} \sum_{i=1}^N \varphi_x(X_i) \xrightarrow{a.s.} \mathbb{E}[\varphi_x(X_1)]
= f_\sigma(x),
\]
Thus, for every fixed $x$, pointwise almost sure convergence holds. We now
upgrade this to uniform convergence over $x\in\mathcal{U}$ using compactness
and equicontinuity.

Since $K_\sigma$ is smooth in $x$ by assumption, and
($\mathcal{U}\times\mathcal{U}$) is compact, the Riemannian gradient of
$K_\sigma$ with respect to the first argument is uniformly bounded:
there exists $L>0$ such that,
\[
\sup_{x,y\in\mathcal{U}} \big\|\nabla_{\mathbb{H}} K_\sigma(x,y)\big\|
\le L < \infty.
\]
By the mean value inequality on Riemannian manifold, this implies that
for all $x,x'\in\mathcal{U}$ and all $y\in\mathcal{U}$,
\begin{equation}\label{eq:kernel-Lip}
\big|K_\sigma(x,y) - K_\sigma(x',y)\big|
\le L\, d_{\mathbb{H}}(x,x').
\end{equation}
In particular, the family $\{\varphi_x : x\in\mathcal{U}\}$ is
equicontinuous in $x$ (uniformly in $y$). Since $\mathcal{U}$ is compact, for every $\delta>0$ there exists a finite
$\delta$-net $\{x_1,\dots,x_m\}\subset\mathcal{U}$ such that for every
$x\in\mathcal{U}$ there is $k\in\{1,\dots,m\}$ with
$d_{\mathbb{H}}(x,x_k)\le\delta$.\\
Fix $\varepsilon>0$. Choose $\delta>0$ such that
\[
L\,\delta < \varepsilon/4.
\]
Let $\{x_1,\dots,x_m\}$ be a finite $\delta$-net of $\mathcal{U}$. For each
$k\in\{1,\dots,m\}$, the strong law of large numbers implies
\[
\hat f_\sigma(x_k) = \frac{1}{N}\sum_{i=1}^N K_\sigma(x_k,X_i)
\xrightarrow{a.s.} f_\sigma(x_k),
\]
as $N\to\infty$. Since $m<\infty$, by union bound \citep{wainwright2019high} and the fact that a finite
union of almost sure events is almost sure, we have
\[
\max_{1\le k\le m} \big|\hat f_\sigma(x_k) - f_\sigma(x_k)\big|
\xrightarrow{a.s.} 0.
\]
Thus, with probability $1$, there exists $N_0$ such that for all
$N\ge N_0$,
\[
\max_{1\le k\le m} \big|\hat f_\sigma(x_k) - f_\sigma(x_k)\big| < \varepsilon/2.
\]
Now fix such an $N$ and consider any $x\in\mathcal{U}$. Choose $k$ such that
$d_{\mathbb{H}}(x,x_k)\le\delta$. Then
$$
\big|\hat f_\sigma(x) - f_\sigma(x)\big|
\le \big|\hat f_\sigma(x) - \hat f_\sigma(x_k)\big|
   + \big|\hat f_\sigma(x_k) - f_\sigma(x_k)\big| 
 + \big|f_\sigma(x_k) - f_\sigma(x)\big|.
$$
We bound each term:

\smallskip\noindent
\emph{(i)} Using \eqref{eq:kernel-Lip},
\begin{align*}
\big|\hat f_\sigma(x) - \hat f_\sigma(x_k)\big|
&= \left|\frac{1}{N}\sum_{i=1}^N
    \big(K_\sigma(x,X_i) - K_\sigma(x_k,X_i)\big)\right| \\
&\le \frac{1}{N}\sum_{i=1}^N
    \big|K_\sigma(x,X_i) - K_\sigma(x_k,X_i)\big| \\
&\le \frac{1}{N}\sum_{i=1}^N L\, d_{\mathbb{H}}(x,x_k)
 \le L\,\delta < \varepsilon/4.
\end{align*}

\smallskip\noindent
\emph{(ii)} By the choice of $N\ge N_0$,
\[
\big|\hat f_\sigma(x_k) - f_\sigma(x_k)\big| < \varepsilon/2.
\]

\smallskip\noindent
\emph{(iii)} Since $K_\sigma$ and $f$ are bounded
and $K_\sigma$ is Lipschitz in $x$ as in \eqref{eq:kernel-Lip}, we obtain
by the Dominated Convergence theorem \citep{rudin},
\begin{align*}
\big|f_\sigma(x) - f_\sigma(x_k)\big| 
& = \left|\int (K_\sigma(x,y)-K_\sigma(x_k,y)) f(y)\,d\mathrm{vol}_{\mathbb{H}}(y)\right| \\
&\le L\,\delta < \varepsilon/4. 
\end{align*}
Combining \textit{(i)}--\textit{(iii)}, we obtain
\[
\big|\hat f_\sigma(x) - f_\sigma(x)\big|
< \varepsilon/4 + \varepsilon/2 + \varepsilon/4 = \varepsilon
\]
for all $x\in\mathcal{U}$ and all $N\ge N_0$, almost surely. Hence
\[
\sup_{x\in\mathcal{U}} \big|\hat f_\sigma(x) - f_\sigma(x)\big|
\xrightarrow{a.s.} 0,
\]
which proves the first claim. \\ \\
The argument for the gradient is analogous. By a similar kind of argument, we can show that :
\[
\sup_{x\in\mathcal{U}} \|\nabla_{\mathbb{H}} \hat f_\sigma(x) - \nabla_{\mathbb{H}} f_\sigma(x)\|
\xrightarrow{a.s.} 0.
\]
Since $\mathcal{U}$ contains the support of $f$ and is compact, the same
suprema over $\mathbb{D}_c^p$ coincide with the suprema over $\mathcal{U}$.
This proves the lemma.
\end{proof}

\begin{lemma}\label{lem:stability}
Suppose $m$ is a non-degenerate critical point of $f_\sigma$, i.e.
\[
\nabla_{\mathbb{H}} f_\sigma(m) = 0, \quad \text{Hess}_{\mathbb{H}} f_\sigma(m) \ \text{is nonsingular}.
\]
Then for large $N$, there exists a unique critical point $\hat m_N$ of $\hat f_\sigma$ such that
\[
d_{\mathbb{H}}(\hat m_N, m) \to 0 \quad a.s.
\]
\end{lemma}

\begin{proof}
By Lemma~\ref{lem:uniform}, the empirical gradient converges uniformly to the
population gradient on compact subset $\mathcal{U}$:
\[
\sup_{x\in \mathcal{U}}
\big\| \nabla_{\mathbb{H}} \hat f_\sigma(x)
      - \nabla_{\mathbb{H}} f_\sigma(x) \big\|
\xrightarrow[N\to\infty]{a.s.} 0.
\]
In particular, this holds on a geodesic ball $B_{\mathbb{H}}(m,r)$ for $r>0$
sufficiently small.
\noindent
Consider normal coordinates $\xi \in T_m\mathcal{M} \cong \mathbb{R}^p$ defined by
$x = \exp_m(\xi)$, and mappings
\begin{align*}
& F(\xi) := \nabla_{\mathbb{H}} f_\sigma(\exp_m(\xi)), \\
& \hat F_N(\xi) := \nabla_{\mathbb{H}} \hat f_\sigma(\exp_m(\xi)). 
\end{align*}
Since $m$ is a nondegenerate critical point, we have $F(0)=0$ and
$DF(0) = \mathrm{Hess}_{\mathbb{H}}f_\sigma(m)$, which is invertible by
assumption. By continuity of the Hessian, there exists $r>0$ such that
$DF(\xi)$ remains invertible for all $\|\xi\|\le r$.
\noindent
By the uniform convergence established above,
\[
\sup_{\|\xi\|\le r} \|\hat F_N(\xi)-F(\xi)\| \to 0 \quad a.s.,
\]
so for sufficiently large $N$, the perturbation
($\hat F_N - F$) is uniformly small on $B(0,r)$.

Since $DF(0)$ is invertible and $F(0)=0$, there exists a unique zero of $F$
inside $B(0,r)$, namely $\xi=0$. The manifold version of the Implicit Function
theorem \citep{lee2018introduction}; for $N$ sufficiently large,
$\hat F_N$ also admits a unique zero $\xi_N$ inside $B(0,r)$ and that this zero
depends continuously on the perturbation $\hat F_N$. Hence,
\[
\xi_N \longrightarrow 0 \quad \text{almost surely}.
\]

Finally, since the Riemannian exponential map is a diffeomorphism in a normal
neighborhood, the convergence $\xi_N\to 0$ implies
\[
\hat m_N = \exp_m(\xi_N) \to \exp_m(0) = m
\quad \text{almost surely},
\]
which completes the proof.
\end{proof}

\begin{theorem} \label{thm:consistency}
Assume the following conditions hold:  
\begin{enumerate}
    \item $f_\sigma$ is continuous and bounded with compact support in $\mathbb{D}^p_c$.
    \item Each mode of $f_\sigma$ is isolated and nondegenerate.
\end{enumerate}
Then for each $m_j \in \mathcal{M}_\sigma$, there exists $\hat m_j \in \hat{\mathcal{M}}_\sigma$ such that
\[
d_{\mathbb{H}}(\hat m_j, m_j) \xrightarrow{\mathbb{P}} 0.
\] that is convergence in probability.\\
Equivalently, in terms of the distance,
\[
d_{\mathbb{H}}(\hat{\mathcal{M}}_\sigma, \mathcal{M}_\sigma) \xrightarrow{\mathbb{P}} 0.
\]
\end{theorem}

\begin{proof}
By assumption, $f_\sigma$ has compact support in $\mathbb{D}_c^p$, say $\mathcal{U}$ so
let
\[
\mathcal{U} := \operatorname{supp}(f_\sigma),
\]
which is compact by definition.
By assumption, $f_\sigma$ has finitely many isolated, nondegenerate modes
$$
\mathcal{M}_\sigma = \{m_1,\dots,m_K\} \subset \mathcal{U},
\qquad 
\nabla_{\mathbb{H}} f_\sigma(m_j) = 0, \qquad
\mathrm{Hess}_{\mathbb{H}} f_\sigma(m_j) \text{ nonsingular}. 
$$
Since each mode is isolated, for each $j$ we choose $r_j>0$ sufficiently small so that the geodesic balls
\[
B_j := B_{\mathbb{H}}(m_j,r_j)
\]
satisfy:
\begin{enumerate}[label=(\roman*)]
    \item The sets $B_j$ are pairwise disjoint.
    \item $m_j$ is the only critical point of $f_\sigma$ in $B_j$.
\end{enumerate}

Define the partition of the compact support:
\[
\mathcal{B} := \bigcup_{j=1}^K (B_j \cap \mathcal{U}),
\qquad
\mathcal{C} := \mathcal{U} \setminus \mathcal{B}.
\]
Since $\mathcal{U}$ is compact and $\mathcal{B}$ is open in the subspace topology of $\mathcal{U}$, 
$\mathcal{C}$ is closed in $\mathcal{U}$ and therefore compact.
As $f_\sigma$ has no critical points in $\mathcal{C}$ and $\nabla_{\mathbb{H}} f_\sigma$ is continuous, the minimum of its gradient norm on $\mathcal{C}$ is strictly positive:
\begin{equation}\label{eq:C-gradient-bound}
\inf_{x\in\mathcal{C}} \|\nabla_{\mathbb{H}} f_\sigma(x)\|
:= 2\delta > 0.
\end{equation}
By Lemma~\ref{lem:uniform}, we have :
\[
\sup_{x \in \mathcal{U}} \|\nabla_{\mathbb{H}} \hat f_\sigma(x) - \nabla_{\mathbb{H}} f_\sigma(x)\| 
\xrightarrow{a.s.} 0.
\]
Therefore, for the $\delta$ from \eqref{eq:C-gradient-bound}, there exists a random $N_1$ such that for all $N\ge N_1$,
\begin{equation}\label{eq:uniform-small}
\sup_{x\in \mathcal{U}} \|\nabla_{\mathbb{H}} \hat f_\sigma(x) - \nabla_{\mathbb{H}} f(x)\| < \delta,
\quad \text{a.s.}
\end{equation}
For $x \in \mathcal{C}$ and $N\ge N_1$, using \eqref{eq:C-gradient-bound} and \eqref{eq:uniform-small},
\[
\|\nabla_{\mathbb{H}} \hat f_\sigma(x)\|
\ge \|\nabla_{\mathbb{H}} f_\sigma(x)\| - \delta
\ge 2\delta - \delta = \delta > 0.
\]
Thus, for large $N$, $\hat f_\sigma$ has no critical points in $\mathcal{C}$. Hence, all critical points of $\hat f_\sigma$ inside $\mathcal{U}$ lie in $\mathcal{B}$.

\noindent Fix $j\in\{1,\dots,K\}$. Since $m_j$ is a nondegenerate critical point, Lemma~\ref{lem:stability} implies that for sufficiently large $N$, there exists a unique critical point $\hat m_j$ of $\hat f_\sigma$ inside $B_j \cap \mathcal{U}$. Thus,
    \[
    d_{\mathbb{H}}(\hat m_j, m_j) \xrightarrow[N\to\infty]{\mathbb{P}} 0.
    \]
Now,
\begin{align*}
   d_{\mathbb{H}}(\hat{\mathcal{M}}_\sigma, \mathcal{M}_\sigma)
:= \max\Big\{
& \sup_{\hat m \in \hat{\mathcal{M}}_\sigma}\inf_{m \in \mathcal{M}_\sigma} d_{\mathbb{H}}(\hat m, m),
\; \\
& \sup_{m \in \mathcal{M}_\sigma}\inf_{\hat m \in \hat{\mathcal{M}}_\sigma} d_{\mathbb{H}}(\hat m, m)
\Big\}. 
\end{align*}
For sufficiently large $N$, each true mode $m_j$ matches uniquely to $\hat m_j$, so
\[
d_{\mathbb{H}}(\hat{\mathcal{M}}_\sigma, \mathcal{M}_\sigma)
=
\max_{1 \le j \le K} d_{\mathbb{H}}(\hat m_j, m_j).
\]
Since $K < \infty$ and each term in the RHS converges in probability to $0$, the maximum also converges in probability to $0$:
\[
d_{\mathbb{H}}(\hat{\mathcal{M}}_\sigma, \mathcal{M}_\sigma)
\xrightarrow[N\to\infty]{\mathbb{P}} 0.
\]
This completes the proof.
\end{proof}

\begin{theorem}
If the bandwidth sequence satisfies
\[
\sigma_N \to 0, \qquad N\sigma_N^p \to \infty,
\]
then the estimated centroids $\hat{\mathcal{M}}_{\sigma_N}$ converge in probability to the true modes of $f$.
\end{theorem}

\begin{proof}
We now let the bandwidth depend on $N$, writing $\sigma_N \to 0$ with
$N\sigma_N^p \to \infty$, and the KDE
\[
  \hat f_{\sigma_N}(x)
  = \frac{1}{N}\sum_{i=1}^n 
    \exp\!\Big(-\tfrac{d_{\mathbb{H}}^2(x,X_i)}{2\sigma_N^2}\Big).
\]
Let $f$ be the true density with compact support $\mathcal{U} := \operatorname{supp}(f)$,
and for each $\sigma>0$ define the population smoothed density
\[
  f_\sigma(x)
  = \int_{\mathbb{D}_c^p} K_\sigma(x,y) f(y)\, d\mathrm{vol}_{\mathbb{H}}(y).
\]

\noindent We assume that $f$ has finitely many isolated, nondegenerate modes
$\mathcal{M} = \{m_1,\dots,m_K\}$ in $\mathcal{U}$, so that each $m_j$ satisfies
\[
  \nabla_{\mathbb{H}} f(m_j) = 0,
  \qquad
  \mathrm{Hess}_{\mathbb{H}} f(m_j) \ \text{is nonsingular}.
\]
The goal is to show that the empirical centroids
$\hat{\mathcal{M}}_{\sigma_N}$ produced by HypeGBMS converge in probability
to the true modes $\mathcal{M}$. \\
\noindent
Consider the gradient error decomposition
$$
      \nabla_{\mathbb{H}} \hat f_{\sigma_N}(x) - \nabla_{\mathbb{H}} f(x)
  = (\nabla_{\mathbb{H}} \hat f_{\sigma_N}(x)
                    - \nabla_{\mathbb{H}} f_{\sigma_N}(x))
  \;+\; 
 (\nabla_{\mathbb{H}} f_{\sigma_N}(x)
                    - \nabla_{\mathbb{H}} f(x)).
$$
We control each term uniformly over the compact support $\mathcal{U}$. \\
\noindent
For each fixed $N$, Lemma~\ref{lem:uniform} (applied with bandwidth $\sigma_N$)
gives the \emph{conditional} uniform convergence
\[
  \sup_{x\in \mathcal{U}}
    \big\|\nabla_{\mathbb{H}} \hat f_{\sigma_N}(x) 
         - \nabla_{\mathbb{H}} f_{\sigma_N}(x)\big\|
  \xrightarrow[N\to\infty]{\mathbb{P}} 0
\]
under the bandwidth condition $N\sigma_N^p \to \infty$. This follows from
classical kernel density estimation theory on manifolds: the effective sample
size in a ball of radius $O(\sigma_N)$ is $N\sigma_N^p$, and the variance term
is of order $(N\sigma_N^p)^{-1/2}$, which vanishes under the assumed condition.
Thus,
\begin{equation}\label{eq:stochastic-grad}
  \sup_{x\in \mathcal{U}}
    \big\|\nabla_{\mathbb{H}} \hat f_{\sigma_N}(x) 
         - \nabla_{\mathbb{H}} f_{\sigma_N}(x)\big\|
  \xrightarrow[N\to\infty]{\mathbb{P}} 0.
\end{equation} \\
\noindent
Because $f$ is continuous with compact support and the kernel
$K_\sigma(x,y) = \exp(-d_{\mathbb{H}}^2(x,y)/(2\sigma^2))$ is smooth and
acts as an approximate identity as $\sigma\to 0$, standard arguments in
Riemannian kernel smoothing \citep{pelletier2005kernel}(using normal coordinates and Taylor expansion
of $f$) imply that
\begin{align*}
    & \sup_{x\in \mathcal{U}} |f_{\sigma_N}(x) - f(x)| \to 0, \\
   & \sup_{x\in \mathcal{U}} \|\nabla_{\mathbb{H}} f_{\sigma_N}(x) - \nabla_{\mathbb{H}} f(x)\|
  \to 0
\end{align*}
as $N\to\infty$ (since $\sigma_N\to 0$). Thus,
\begin{equation}\label{eq:bias-grad}
  \sup_{x\in \mathcal{U}}
    \big\|\nabla_{\mathbb{H}} f_{\sigma_N}(x)
         - \nabla_{\mathbb{H}} f(x)\big\|
  \xrightarrow[N\to\infty]{} 0.
\end{equation} \\
\noindent
Combining \eqref{eq:stochastic-grad} and \eqref{eq:bias-grad} via the triangle
inequality, we obtain
$$
    \sup_{x\in \mathcal{U}}
    \big\|\nabla_{\mathbb{H}} \hat f_{\sigma_N}(x) - \nabla_{\mathbb{H}} f(x)\big\| 
  \le 
   \sup_{x\in \mathcal{U}} \big\|\nabla_{\mathbb{H}} \hat f_{\sigma_N}(x)
                 - \nabla_{\mathbb{H}} f_{\sigma_N}(x)\big\| 
   + \sup_{x\in \mathcal{U}} \big\|\nabla_{\mathbb{H}} f_{\sigma_N}(x)
                  - \nabla_{\mathbb{H}} f(x)\big\|. 
$$
\noindent
Thus, the empirical gradient field $\nabla_{\mathbb{H}} \hat f_{\sigma_N}$
converges uniformly in probability to the true gradient field
$\nabla_{\mathbb{H}} f$ on the compact set $\mathcal{U}$:
\begin{equation}\label{eq:unif-grad-to-f}
  \sup_{x\in \mathcal{U}}
    \big\|\nabla_{\mathbb{H}} \hat f_{\sigma_N}(x) - \nabla_{\mathbb{H}} f(x)\big\|
  \xrightarrow[N\to\infty]{\mathbb{P}} 0.
\end{equation} \\
\noindent
As in the proof of Theorem~\ref{thm:consistency} (fixed bandwidth case),
we exploit the nondegeneracy of the modes of $f$.
\noindent
For each mode $m_j\in\mathcal{M}$, choose a radius $r_j>0$ such that the
balls
\[
  B_j := B_{\mathbb{H}}(m_j,r_j)
\]
are pairwise disjoint and each $B_j$ contains no other critical point of $f$
than $m_j$. Define
\[
  \mathcal{B} := \bigcup_{j=1}^K (B_j\cap \mathcal{U}),
  \qquad
  \mathcal{C} := \mathcal{U} \setminus \mathcal{B}.
\]
As before, $\mathcal{C}$ is compact, contains no critical points of $f$, and
by continuity of $\nabla_{\mathbb{H}} f$ there exists $\delta>0$ such that
\begin{equation}\label{eq:C-grad-bound-f}
  \inf_{x\in\mathcal{C}} \|\nabla_{\mathbb{H}} f(x)\|
  \ge 2\delta > 0.
\end{equation}
By \eqref{eq:unif-grad-to-f}, for any $\varepsilon>0$,
\[
  \mathbb{P}\!\left(
    \sup_{x\in \mathcal{U}} \|\nabla_{\mathbb{H}} \hat f_{\sigma_N}(x) 
                   - \nabla_{\mathbb{H}} f(x)\|
    > \varepsilon
  \right)
  \to 0.
\]
Take $\varepsilon=\delta$. Then with probability tending to $1$ as $N\to\infty$,
we have
\[
  \sup_{x\in \mathcal{U}} \|\nabla_{\mathbb{H}} \hat f_{\sigma_N}(x) 
                 - \nabla_{\mathbb{H}} f(x)\|
  < \delta.
\]
On this event, for all $x\in\mathcal{C}$,
\[
  \|\nabla_{\mathbb{H}} \hat f_{\sigma_N}(x)\|
  \ge \|\nabla_{\mathbb{H}} f(x)\| - \delta
  \ge 2\delta - \delta = \delta > 0,
\]
using \eqref{eq:C-grad-bound-f}. Thus, with probability tending to $1$, the
empirical density $\hat f_{\sigma_N}$ has no critical point in $\mathcal{C}$,
and all its critical points in $\mathcal{U}$ must lie inside $\mathcal{B}$.
\noindent Near each $m_j$, the nondegenerate critical point property for $f$ and
the uniform convergence \eqref{eq:unif-grad-to-f} allow us to invoke the same
stability argument as in Lemma~\ref{lem:stability} (now with $f$ as the
population target and $\hat f_{\sigma_N}$ as the empirical estimator). This
shows:

\begin{enumerate}[label=(\roman*)]
  \item For each $j$, with probability tending to $1$, there exists a unique
  critical point $\hat m_j^N$ of $\hat f_{\sigma_N}$ inside $B_j$.
  \item Moreover,
  \[
    d_{\mathbb{H}}(\hat m_j^N, m_j) \xrightarrow[N\to\infty]{\mathbb{P}} 0.
  \]
\end{enumerate} 
Let $\hat{\mathcal{M}}_{\sigma_N}$ denote the set of all critical points of
$\hat f_{\sigma_N}$ in $\mathcal{U}$.
By the previous step, with probability tending to $1$ we have,
\[
  \hat{\mathcal{M}}_{\sigma_N}
  = \{\hat m_1^N,\dots,\hat m_K^N\},
\]
with each $\hat m_j^N$ lying in $B_j$ and converging in probability to $m_j$.

\noindent Define the Hausdorff distance between the finite sets
$\hat{\mathcal{M}}_{\sigma_N}$ and $\mathcal{M}$ by,
\begin{align*}
      d_{\mathbb{H}}(\hat{\mathcal{M}}_{\sigma_N},\mathcal{M})
  := \max\Big\{
   & \sup_{\hat m\in\hat{\mathcal{M}}_{\sigma_N}}
      \inf_{m\in\mathcal{M}} d_{\mathbb{H}}(\hat m,m),\; \\
   & \sup_{m\in\mathcal{M}}
      \inf_{\hat m\in\hat{\mathcal{M}}_{\sigma_N}}
      d_{\mathbb{H}}(\hat m,m)
  \Big\}.
\end{align*}

\[
  d_{\mathbb{H}}(\hat{\mathcal{M}}_{\sigma_N},\mathcal{M})
  = \max_{1\le j\le K} d_{\mathbb{H}}(\hat m_j^N, m_j).
\]
Since $K<\infty$ and $d_{\mathbb{H}}(\hat m_j^N, m_j)\to 0$ in probability for
each $j$, it follows by a union bound \citep{wainwright2019high} that
\[
  d_{\mathbb{H}}(\hat{\mathcal{M}}_{\sigma_N},\mathcal{M})
  \xrightarrow[N\to\infty]{\mathbb{P}} 0.
\]

\noindent This shows that, under the bandwidth conditions $\sigma_N\to 0$ and
$N\sigma_N^p\to\infty$, the estimated centroids
$\hat{\mathcal{M}}_{\sigma_N}$ converge in probability to the true set of
modes of $f$. Since HypeGBMS converges algorithmically to empirical modes, the centroids are statistically consistent for the true population modes.
\end{proof}


\section{Experiments} \label{experiments}
\subsection{Details of the Datasets.}
We validate the efficacy of HypeGBMS on $11$ real-world datasets. Iris, Glass, Ecoli, Wine, Wisconsin B.C, Phishing URL, Abalone, Glass, Zoo, ORHD (Optical Recognition of Handwritten Digits) datasets are taken from the UCI machine learning repository \citep{dua2017uci}; Flights, MNIST, Pendigits, ORL face datasets are taken from Kaggle.

\begin{table*}[h]
\centering
\caption{Comparison of Clustering Performance across multiple methods, $k$-means, GMS, GBMS, DBSCAN, Spectral, GMM, QuickMeanShift, WBMS, Grid-shift, HSFC with our proposed HypeGBMS on $11$ Real-world datasets, presented. The best and second-best results are highlighted in boldface and underlined, respectively.}
\label{tab:real-dataset}
\resizebox{\columnwidth}{!}{\begin{Huge}
\scriptsize
\begin{tabular}{llcccccccccccccc}
\toprule
\textbf{Dataset} & \textbf{Metric} & \textbf{$k$-means} & \textbf{GMS} & \textbf{GBMS} & \textbf{GMM} & \textbf{DBSCAN}  & \textbf{Spectral} & \textbf{QuickMeanShift} & \textbf{WBMS} & \textbf{Gridshift} & \textbf{HSFC} & \textbf{HypeGBMS (Ours)} \\
\midrule
\multirow{2}{*}{Iris} 
& ARI & \underline{0.742} & 0.563 & 0.568 & 0.507 & 0.518 & 0.418 & 0.701 & 0.568 & 0.714 & 0.621 & \textbf{0.755} \\
& NMI & \underline{0.767} & 0.717 & 0.734 & 0.614 & 0.626 & 0.509 & 0.712 & 0.733 & 0.754 & 0.663 & \textbf{0.794} \\
\midrule 
\multirow{2}{*}{Glass} 
& ARI & 0.168 & 0.187 & \underline{0.261} & 0.205 & 0.225 & 0.142 & 0.256 & 0.157 & 0.234 & 0.258 & \textbf{0.281} \\
& NMI & 0.306 & 0.361 & \underline{0.438} & 0.361 & 0.358 & 0.265 & 0.427  & 0.238 & 0.314 & 0.428 & \textbf{0.477} \\
\midrule
\multirow{2}{*}{Ecoli} 
& ARI & 0.384 & 0.661 & \underline{0.669} & 0.637 & 0.435 & 0.391 & 0.354 & 0.038 & 0.484 & 0.431 & \textbf{0.673} \\
& NMI & 0.534 & 0.626 & \underline{0.631} & 0.614 & 0.437 & 0.592 & 0.513 & 0.112 & 0.515 & 0.565 & \textbf{0.635} \\
\midrule
\multirow{2}{*}{Wine} 
& ARI & 0.352 & 0.103 & 0.562 & 0.398 & 0.329 & \textbf{0.881} & 0.166 & 0.802 & 0.216 & 0.371 & \underline{0.813} \\
& NMI & 0.423 & 0.346 & 0.588 & 0.585 & 0.419 & \textbf{0.860} & 0.319 & 0.795 & 0.354 & 0.429 & \underline{0.838} \\
\midrule
\multirow{2}{*}{Wisconsin B.C.} 
& ARI & 0.244 & 0.667 & 0.725 & 0.265 & 0.297 & \underline{0.871} & 0.644 & 0.679 & 0.454 & 0.718 & \textbf{0.882} \\
& NMI & 0.402 & 0.469 & 0.661 & 0.362 & 0.456 & \underline{0.783} & 0.546 & 0.584 & 0.473 & 0.742 & \textbf{0.801} \\
\midrule
\multirow{2}{*}{Zoo} 
& ARI & 0.714 & 0.544 & \underline{0.794} & 0.674 & 0.515 & 0.513 & 0.338 & 0.751 & 0.442 & 0.499 & \textbf{0.807} \\
& NMI & 0.771 & 0.614 & \underline{0.821} & 0.789 & 0.678 & 0.746 & 0.526 & 0.784 & 0.578 & 0.717 & \textbf{0.846} \\
\midrule
\multirow{2}{*}{Phishing (5K)} 
& ARI & 0.001 & 0.441 & 0.499 & 0.392 & 0.005 & \underline{0.625} & 0.224 & 0.102 & 0.374 & 0.121 & \textbf{0.921} \\
& NMI & 0.002 & 0.594 & 0.401 & 0.543 & 0.011 & \underline{0.597} & 0.371 & 0.128 & 0.387 & 0.145 & \textbf{0.866} \\
\midrule
\multirow{2}{*}{MNIST (5K)} 
& ARI & 0.235 & 0.361 & 0.142 & 0.269 & 0.141 & \underline{0.392} & 0.314 & 0.001 & 0.218 & 0.361 & \textbf{0.584} \\
& NMI & 0.381 & 0.465 & 0.165 & 0.411 & 0.278 & \underline{0.503} & 0.401 & 0.002 & 0.294 & 0.489 & \textbf{0.701} \\
\midrule
\multirow{2}{*}{ORHD} 
& ARI & 0.281 & 0.342 & 0.517 & 0.358 & 0.109 & \underline{0.535}  & 0.155  & 0.001 & 0.164 & 0.257 & \textbf{0.593} \\
& NMI & 0.334 & 0.421 & 0.661 & 0.561 & 0.259 & \underline{0.632} & 0.358 & 0.004 & 0.315 & 0.551 & \textbf{0.702} \\
\midrule
\multirow{2}{*}{ORL} 
& ARI & 0.341 & 0.413 & 0.517 & 0.358 & 0.406 & 0.535 & 0.427 & 0.361 & 0.315 & \underline{0.561} & \textbf{0.568} &  \\
& NMI & 0.385 & 0.512 & 0.591 & 0.561 & 0.484 & 0.632 & 0.628 & 0.432 & 0.377 & \underline{0.641} & \textbf{0.653} \\
\midrule
\multirow{2}{*}{Pendigits} 
& ARI & 0.538 & 0.499 & 0.577 & \underline{0.639} & 0.158 & 0.580 & 0.257 & 0.121 & 0.253 & 0.523 & \textbf{0.667} \\
& NMI & 0.673 & 0.681 & 0.714 & \underline{0.754} & 0.312 & 0.717 & 0.393 & 0.336 & 0.402 & 0.706 & \textbf{0.776} \\
\bottomrule
\end{tabular}
\end{Huge}}

\end{table*}

\subsection{Experimental Setup \& Baselines.}
The performance of HypeGBMS is evaluated using two well-known performance metrics: the Adjusted Rand Index (ARI) \citep{ari} and the Normalized Mutual Information (NMI) \citep{nmi}. The notable base methods, namely $k$-means \citep{kmeans}, Gaussian Mean Shift (GMS) \citep{mean}, GBMS \citep{fastgbms}, GMM \citep{gmm}, DBSCAN \citep{dbscan}, Spectral \citep{spectral}, along with the State-of-the-Art methods like QuickMeanshift \citep{quickmeanshift}, WBMS \citep{wbms}, GridShift \citep{gridshift}, HSFC \citep{hsfc} clustering methods, are considered for comparison with our proposed HypeGBMS method. A brief description of the data sets used for the experimentation is given below. 

We demonstrate the effectiveness of the proposed framework using the Berkeley Segmentation Data Set (BSDS500) \citep{martinbsds}, and the PASCAL VOC 2012 dataset \citep{everingham}. We utilize the four metrics: Segmentation Covering (SC), Probabilistic Rand Index (PRI) \citep{pri}, Variation of Information (VoI) \citep{voi}, 
and F1-Score \citep{f1score} to compare the results quantitatively. Considering the SC, PRI, and F1-score, higher values represent better results. For VoI, lower values denote better segmentation.

\begin{table*}[ht]
\centering
\caption{Quantitative Results on the BSDS500 \citep{martinbsds} dataset across multiple methods, $k$-means, GMS, GBMS, DBSCAN, Spectral, GMM, QuickMeanShift, WBMS, Grid-shift, HSFC with our proposed HypeGBMS.}
\label{tab:bsds500}
\scriptsize
\begin{tabular}{llcccc}
\toprule
Methods & Venue & SC & Vol & F1-Score \\
\midrule
$k$-means (\citet{kmeans}) & [TPAMI'02]   & 0.1785 & 2.6851 & 0.1611 \\
GMS (\citet{mean})  & [TPAMI'95] & 0.1968 & 3.0024 & 0.1636 \\
GBMS (\citet{gbms}) & [ICML'06] & 0.1984 &  
2.9714 & 0.1825 \\
GMM (\citet{gaussian})  & [Encycl. Biom'09] & 0.2165 & 2.2932 & 0.2106 \\
DBSCAN (\citet{dbscan}) & [ICADIWT'14] & 0.2247 & 2.1643 & 0.2993 \\
Spectral (\citet{spectral}) & [Stat. Comp'07] & 0.2081 & 2.2521 & 0.2611 \\
WBMS (\citet{wbms})& [AAAI'21]  & 0.1833 & 2.6668 & 0.3129 \\
QuickMeanShift (\citet{quickmeanshift}) & [AAAI'23]  & 0.1665 & 3.0456 & 0.2784 \\
Gridshift (\citet{gridshift}) & [CVPR'22] & 0.2016  & 2.9814 & 0.2997 \\
HSFC (\citet{hsfc}) & [PR'10] & 0.2036 & 2.2341 & 0.2774 \\
\midrule
\textbf{HypeGBMS (Ours)} &  & \textbf{0.2305} & \textbf{2.0441} & \textbf{0.3358} \\
\bottomrule
\end{tabular}
\end{table*}

\subsection{Experiment on Datasets.} We carried out our experiments on $11$ Real-world datasets. The performance noted in Table \ref{tab:real-dataset} shows that our proposed method, HypeGBMS, outperforms almost all competitors in terms of ARI and NMI. 

\begin{table*}[h]
\centering
\scriptsize
\caption{Quantitative Results on PASCAL VOC 2012 \citep{everingham} dataset across multiple methods, $k$-means, GMS, GBMS, DBSCAN, Spectral, GMM, QuickMeanShift, WBMS, Grid-shift, HSFC with our proposed HypeGBMS.}
\label{tab:pascalvoc}
\scriptsize
\begin{tabular}{llccc}
\toprule
Methods & Venue & PRI & Vol & F1-Score \\
\midrule
$k$-means (\citet{kmeans}) & [TPAMI'02] & 0.4001 & 0.9285 & 0.5845 \\
GMS (\citet{mean}) & [TPAMI'95] & 0.3755 & 1.1951 & 0.5969 \\
GBMS (\citet{gbms}) & [ICML'06] & 0.3977 & 1.2514 & 0.6124 \\
GMM (\citet{gmm}) & [Encycl. Biom'09] & 0.4019 & 0.9087 & 0.6187 \\
DBSCAN (\citet{dbscan}) & [ICADIWT'14] & 0.4112 & 0.9142 & 0.5832 \\
Spectral (\citet{spectral}) & [Stat Comp'07] & 0.3971 & 0.9566 & 0.6194 \\
WBMS (\citet{wbms}) & [AAAI'21] & 0.3881 & 1.0231 & 0.5158 \\
QuickMeanShift (\citet{quickmeanshift}) & [AAAI'23] & 0.3645 & 1.1456  & 0.5874 \\
Gridshift (\citet{gridshift}) & [CVPR'22] & 0.4015 & 1.1671 & 0.6148 \\
HSFC (\citet{hsfc}) & [PR'10] & 0.3614 & 1.2045 & 0.5236 \\
\midrule
\textbf{HypeGBMS (Ours)} & & \textbf{0.4296} & \textbf{0.8848} & \textbf{0.6388} \\
\bottomrule
\end{tabular}
\end{table*}

We have evaluated our proposed method, HypeGBMS, on the BSDS500 \& PASCAL VOC 2012 datasets and compared its performance with the other contender methods. We use SC, VoI, F1-Score for the BSDS500 dataset and PRI, VoI, and F1-Score for the PASCAL VOC 2012 dataset. Table \ref{tab:bsds500} shows the quantitative results on the BSDS500 dataset. Table \ref{tab:pascalvoc} shows the quantitative performance of different indices on the methods in the PASCAL VOC 2012 dataset. The performances in both tables demonstrate that our proposed method, HypeGBMS, outperforms almost all its contenders in terms of SC, PRI, VoI, and F1-Score.

The experimental results on the Phishing dataset in Table \ref{tab:performance} clearly show that HypeGBMS outperforms all other clustering methods. HypeGBMS consistently achieves the highest ARI and NMI scores across both data scales (1k and 5k), demonstrating its strong ability to capture the underlying cluster structure. 
Traditional mean-shift variants such as GMS, GBMS, QuickMeanShift, and Grid Shift show moderate to low accuracy despite reasonable computation times, underscoring their limitations on this dataset. Hyperbolic and graph-based methods like WBMS, and HSFC tend to be less accurate or slower. In contrast, HypeGBMS provides an excellent balance, achieving the best clustering quality while maintaining computational efficiency, making it the most effective method among all these methods.

\subsection{Experiments with Image Segmentation.}
We have performed experiments on our proposed method, HypeGBMS, with other contenders for image segmentation on some images of the BSDS500\citep{martinbsds} datasets, shown in the Figure \ref{fig:bsds500}. Similarly, we also done the same experiment on the PASCAL VOC 2012\citep{everingham} dataset, shown in the Figure \ref{fig:pascal-voc}. 
Comparative analysis was performed using benchmark datasets, where we carefully measured segmentation accuracy and robustness under varying conditions. Across all evaluation metrics, our method consistently outperformed competing techniques, demonstrating superior boundary preservation, higher segmentation accuracy, and improved adaptability to complex structures. These results clearly highlight the strength of our approach and its potential to advance the state-of-the-art methods in image segmentation tasks.

\begin{figure}[h]
    \centering
    \includegraphics[width= 0.55\linewidth]{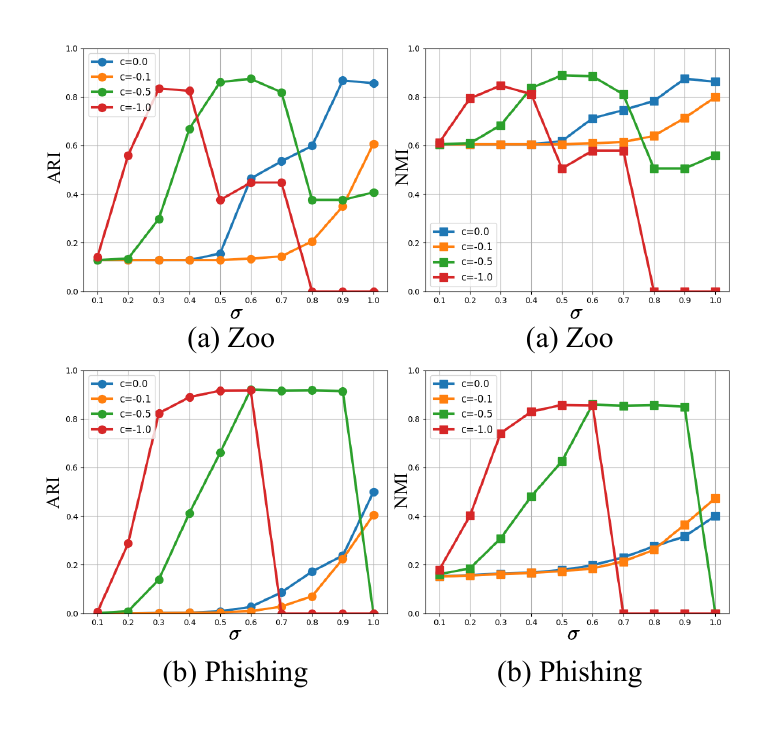}
    \caption{The performance metrics (ARI, NMI) vs bandwidth ($\sigma$) plots for different values of curvature (c) for our proposed HypeGBMS for two real-world datasets: (a) Zoo, (b) Phishing URL, respectively.}
    \label{fig:ablation}
\end{figure}

\begin{table*}[h]
\caption{Clustering performance comparison with runtime on Phishing dataset(for $1k$ and $5k$ samples).}
\label{tab:performance}
\centering
\scriptsize
\resizebox{0.6\columnwidth}{!}{
\begin{tabular}{c|c|c|c|c|c}
\toprule
\textbf{Dataset} & \textbf{Method} & \textbf{N} & \textbf{Time (sec)} & \textbf{ARI} & \textbf{NMI} \\
\midrule
\multirow{14}{*}{\textbf{Phishing}} 
& \multirow{2}{*}{GBMS}       & 1k & 0.216 & 0.514 & 0.464 \\
&            & 5k & 1.011 & 0.499 & 0.401 \\
\cmidrule{2-6}
& \multirow{2}{*}{GMS}        & 1k & 0.326 & 0.492 & 0.622 \\
&            & 5k & 1.263 & 0.441 & 0.594 \\
\cmidrule{2-6}
& \multirow{2}{*}{QuickMeanShift} & 1k & 0.455 & 0.286 & 0.423 \\
&            & 5k & 1.994 & 0.224 & 0.371 \\
\cmidrule{2-6}
& \multirow{2}{*}{Grid Shift} & 1k & 0.664 & 0.406 & 0.414 \\
&            & 5k & 2.624 & 0.374 & 0.387 \\
\cmidrule{2-6}
& \multirow{2}{*}{WBMS}       & 1k & 1.124 & 0.141 & 0.186 \\
&            & 5k & 4.821 & 0.102 & 0.128 \\
\cmidrule{2-6}
& \multirow{2}{*}{HSFC}       & 1k & 0.471 & 0.148 & 0.197 \\
&            & 5k & 2.254 & 0.121 & 0.145 \\
\cmidrule{2-6}
    & \multirow{2}{*}{\textbf{HypeGBMS (Ours)}}   & 1k & 0.464 & \textbf{0.912} & \textbf{0.842} \\
&            & 5k & 2.124 & \textbf{0.921} & \textbf{0.866} \\
\bottomrule
\end{tabular}}
\end{table*}

\subsection{Ablation Study} We conducted an ablation study by varying the two key parameters (i) the bandwidth of the Gaussian kernel ($\sigma$), (ii) the curvature of the Poincaré disc ($c$).\\

\vspace{3pt}
\noindent
\textbf{Choice of bandwidth of Gaussian Kernel.} We have experimented on the performance metrics, ARI, NMI, with the bandwidth parameter ($\sigma$) varying from $0.1$ to $1.0$ for the different curvature values (c) $\{0.0, -0.1, -0.5, -1.0\}$ for two real datasets, Zoo and Phishing URL, respectively in Figure \ref{fig:ablation}. The plots show that for higher curvature values ($c$) such as $c = -0.5, -1.0$; optimal performance is achieved for bandwidth parameters ($\sigma$) lying within the range of $0.4$ to $0.6$.

\begin{figure}[h]
    \centering
    \includegraphics[width= 0.55\linewidth]{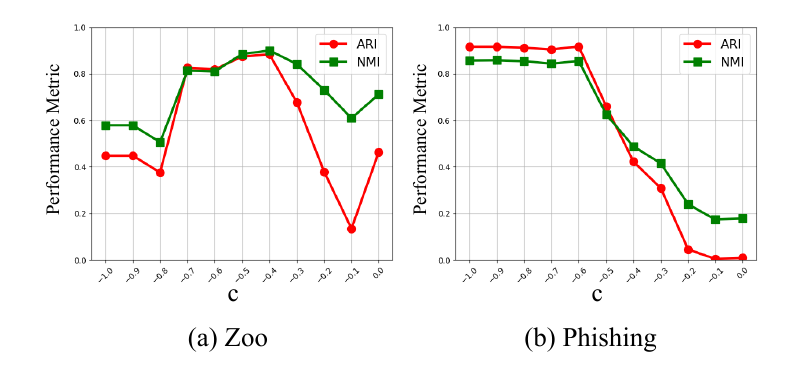}
    \caption{The Performance metrics, ARI, NMI vs curvature ($c$) for our proposed method HypeGBMS for two real-world datasets (a) Zoo, (b) Phishing URL, respectively.}
    \label{fig:curvature ablation}
\end{figure}

\begin{table}[h]
\centering
\caption{ARI for our proposed method, \textit{HypeGBMS} trained on Phishing dataset with different curvature $c$.}
\label{tab:curvature-ablation}
\begin{tabular}{|c|c|c|c|c|c|}
\hline
Curvature (c) & -0.1 & -0.2 & -0.3 & -0.4 & -0.5 \\
\hline
ARI  & 0.084 & 0.109 & 0.358 & 0.426 & 0.751 \\
\hline
Curvature (c) & -0.6 & -0.7 & -0.8 & -0.9 & -1.0 \\
\hline
ARI  & \textbf{0.921} & 0.886  & 0.918  & 0.911 & 0.894 \\
\hline
\end{tabular}
\end{table}

\noindent
\textbf{Choice of Curvature of the Poincaré disc.} We have performed an experiment on the performance metrics (ARI, NMI) with the curvature ($c$) varying from $-1.0$ to $0.0$, fixing the bandwidth ($\sigma$) value $0.6$ on the two real datasets, Zoo and Phishing URL, respectively shown in Figure \ref{fig:curvature ablation}. Table \ref{tab:curvature-ablation} shows the values for the performance metric, ARI, for different curvatures from $-1.0$ to $-0.1$ on the Phishing dataset. The optimal performance is achieved when the curvature value ($c$) is equal to $-0.6$.

\section{Conclusion \& Future Works} \label{conclusion}
Our proposed method, HypeGBMS, a hyperbolic extension of the Gaussian Blurring Mean Shift algorithm, is designed to overcome the limitations of Euclidean clustering methods on hierarchical and tree-structured data. By embedding data in the Poincaré ball and employing Möbius weighted means for mode updates, our method preserves the intrinsic geometry of hyperbolic space while retaining the density-seeking behavior of mean-shift clustering. Theoretical analysis established convergence properties and computational feasibility, while empirical evaluations confirmed that HypeGBMS provides superior performance on datasets where Euclidean approaches struggle. Beyond clustering accuracy, the ability of hyperbolic geometry to efficiently represent hierarchical relationships highlights the broader potential of our framework for tasks in complex network analysis and representation learning. 

Extending HypeGBMS to large-scale datasets through efficient approximation schemes and parallelization strategies would further broaden its applicability. Another promising direction is the development of adaptive mechanisms for selecting curvature parameters, integrating the method with modern hyperbolic deep learning architectures.

\appendix
\section{Appendix}
This Appendix provides the necessary proofs for the Convergence Analysis in
Section~\ref{convergence} of our proposed
HypeGBMS method. Let $(\mathcal{M},g)=(\mathbb{D}^p_c,g_c)$ be the $p$-dimensional Poincaré ball with sectional curvature $c=-\kappa<0$ and hyperbolic distance $d_{\mathbb{H}}$.

\subsection{Comparison bound}

The key geometric input is a pointwise bound on the Hessian of the
squared-distance function. Define, for $s\ge 0$,
\[
g_c(s) := \begin{cases}
\sqrt{\kappa}\, s\,\coth(\sqrt{\kappa}\, s), & \kappa>0\ (\text{i.e. } c<0),\\[4pt]
1, & \kappa=0.
\end{cases}
\]
Note $g_c(s)\ge 1$ for $s\ge 0$, and $g_c(s)\to 1$ as $\kappa\to 0$.

\begin{lemma}[Distance-square Hessian bound]
\label{lem:dist-hessian}
For any fixed $y\in\mathcal{U}$ the function $\phi_y(x):=\tfrac12 d_\mathbb{H}^2(x,y)$
is $\mathbb{C}^2$ on $\mathcal{U}\setminus\{y\}$ and for every $x\in\mathcal{U}$
and tangent vector $v\in T_x\mathcal{M}$,
\[
\mathrm{Hess}_x\big(\tfrac12 d_\mathbb{H}^2(\cdot,y)\big)[v,v] \le
   g_c\big(d_\mathbb{H}(x,y)\big)\,\|v\|_g^2.
\]
Consequently, for any $u\in T_x\mathcal{M}$ small enough that
$\exp_x(u)\in\mathcal{U}$,
\begin{equation}\label{eq:dist-expansion}
d_\mathbb{H}^2\!\big(\exp_x(u),y\big)
\le d_\mathbb{H}^2(x,y) - 2\langle u, \log_x y\rangle + G \,\|u\|_g^2,
\end{equation}
where $G := \sup_{s\in[0,D]} g_c(s) < \infty$.
\end{lemma}

\begin{proof}
Since $(\mathcal{M},g)$ is a Hadamard manifold, the exponential map at any point is a global
diffeomorphism, and there is no cut locus. In particular, the distance
function $x\mapsto d_\mathbb{H}(x,y)$ is smooth on $\mathcal{M}\setminus\{y\}$.
Therefore $\phi_y$ is $\mathbb{C}^2$ on $\mathcal{U}\setminus\{y\}$.

\noindent Let $r(x):=d_\mathbb{H}(x,y)$ denote the hyperbolic distance to $y$. On a manifold
with sectional curvature bounded below by $-\,\kappa$, the Hessian
comparison theorem (see, e.g., \citep[Ch.~5]{boumal2020intromanifold}) states
that for all $x\neq y$ and $v\in T_x\mathcal{M}$,
\begin{equation} \label{eq:hess}
\mathrm{Hess}_x r[v,v]
\;\le\;
\sqrt{\kappa}\,\coth\big(\sqrt{\kappa}\,r(x)\big)
\bigl(\|v\|_g^2 - \langle \nabla r(x), v\rangle^2\bigr),
\end{equation}
with the convention that for $\kappa=0$ the right-hand side is simply
$\|v\|_g^2 - \langle \nabla r(x), v\rangle^2$ (since then $r$ is
Euclidean distance). Here $\nabla r(x)$ denotes the Riemannian gradient
of $r$ at $x$, a unit vector pointing along the geodesic from $x$ to $y$.

\noindent We now express, $\mathrm{Hess}_x(\phi_y)$ in terms of $r$ and $\mathrm{Hess}_x r$.
Recall that
\[
\phi_y(x) = \tfrac12 r(x)^2.
\]
Thus, for any $v\in T_x\mathcal{M}$,
\begin{align*}
\mathrm{Hess}_x \phi_y[v,v]
&= \mathrm{Hess}_x\Big(\tfrac12 r^2\Big)[v,v]\\
&= \tfrac12 \,\mathrm{Hess}_x(r^2)[v,v]\\
&= \tfrac12 \bigl(2\,r(x)\,\mathrm{Hess}_x r[v,v]
       + 2\,\langle \nabla r(x), v\rangle^2\bigr)\\
&= r(x)\,\mathrm{Hess}_x r[v,v] + \langle \nabla r(x), v\rangle^2.
\end{align*}
Using \eqref{eq:hess}, for $\kappa>0$ we obtain
\begin{align*}\mathrm{Hess}_x \phi_y[v,v]
&\le r(x)\,\sqrt{\kappa}\,\coth\big(\sqrt{\kappa}\,r(x)\big)
    \bigl(\|v\|_g^2 - \langle \nabla r(x), v\rangle^2\bigr) 
+ \langle \nabla r(x), v\rangle^2\\
&= g_c\big(r(x)\big)\,\|v\|_g^2
   + \bigl(1 - g_c\big(r(x)\big)\bigr)\,\langle \nabla r(x), v\rangle^2,
\end{align*}
where $g_c(s) = \sqrt{\kappa}\,s\,\coth(\sqrt{\kappa}\,s)$ for $s\ge0$.
Since $g_c(s)\ge 1$ for all $s\ge0$, the coefficient
$1 - g_c\big(r(x)\big)\le 0$, hence the second term is nonpositive. Thus,
\[
\mathrm{Hess}_x \phi_y[v,v]
\;\le\; g_c\big(r(x)\big)\,\|v\|_g^2
= g_c\big(d_\mathbb{H}(x,y)\big)\,\|v\|_g^2.
\]
For $\kappa=0$, we have $r$ equal to Euclidean distance, 
$\mathrm{Hess}_x r[v,v] = \|v\|_g^2 - \langle\nabla r(x),v\rangle^2$, and 
the same calculation yields
$\mathrm{Hess}_x \phi_y[v,v]\le \|v\|_g^2 = g_c(r(x))\|v\|_g^2$ with $g_c\equiv 1$.
This establishes the Hessian bound
\begin{equation} \label{eq:hessian}
\mathrm{Hess}_x\big(\tfrac12 d_\mathbb{H}^2(\cdot,y)\big)[v,v] 
\le g_c\big(d_\mathbb{H}(x,y)\big)\,\|v\|_g^2.
\end{equation}
Let $x\in\mathcal{U}$ and $u\in T_x\mathcal{M}$ be such that 
$\exp_x(u)\in\mathcal{U}$. Consider the unit-speed geodesic
\[
\gamma:[0,1]\to\mathcal{M},\qquad
\gamma(t) := \exp_x(tu),
\]
so that $\gamma(0)=x$, $\gamma(1)=\exp_x(u)$ and 
$\dot\gamma(t)=u$ for all $t$ (with $\|u\|_g$ constant).
Define the scalar function
\[
\psi(t) := \phi_y(\gamma(t))
= \tfrac12 d_\mathbb{H}^2(\gamma(t),y).
\]
Then
\begin{align*}
&\psi'(t) = \big\langle \nabla \phi_y(\gamma(t)),\, \dot\gamma(t)\big\rangle, \\
&\psi''(t) = \mathrm{Hess}_{\gamma(t)}\phi_y[\dot\gamma(t),\dot\gamma(t)].
\end{align*}
Using the Hessian bound from \eqref{eq:hessian} and the fact that 
$\|\dot\gamma(t)\|_g = \|u\|_g$, we obtain
\[
\psi''(t)
\le g_c\big(d_\mathbb{H}(\gamma(t),y)\big)\,\|u\|_g^2.
\]
Let $D:=\sup\{d_\mathbb{H}(z,y):z\in\mathcal{U}\}$; since $\mathcal{U}$ is a
bounded geodesic ball, $D<\infty$, and we may define
\[
G := \sup_{s\in[0,D]} g_c(s) < \infty.
\]
Then, for all $t\in[0,1]$,
\[
\psi''(t) \le G\,\|u\|_g^2.
\]
Integrating this differential inequality twice yields a second-order upper
bound. Integrate from $0$ to $1$:
\[
\psi'(1) - \psi'(0)
= \int_0^1 \psi''(t)\,dt
\le G\,\|u\|_g^2.
\]
Integrating once more,
\begin{align*}
\psi(1) - \psi(0)
&= \int_0^1 \psi'(t)\,dt \\
&= \psi'(0) + \int_0^1 (1-t)\,\psi''(t)\,dt
\le \psi'(0) + \frac{G}{2}\,\|u\|_g^2.
\end{align*}
We compute $\psi'(0)$. Using 
$\phi_y(x)=\tfrac12 d_\mathbb{H}^2(x,y)$ and the standard identity
$\nabla_x \phi_y(x) = -\log_x y$, we obtain
\[
\psi'(0)
= \big\langle \nabla \phi_y(x), u\big\rangle
= -\big\langle \log_x y, u\big\rangle.
\]
Therefore,
\[
\psi(1)
\le \psi(0) - \langle \log_x y, u\rangle + \frac{G}{2}\,\|u\|_g^2.
\]
Recalling that $\psi(t) = \tfrac12 d_\mathbb{H}^2(\gamma(t),y)$ and
$\gamma(1)=\exp_x(u)$, this inequality becomes
\[
\tfrac12 d_\mathbb{H}^2\big(\exp_x(u),y\big)
\le \tfrac12 d_\mathbb{H}^2(x,y)
   - \langle \log_x y, u\rangle
   + \frac{G}{2}\,\|u\|_g^2.
\]
Multiplying both sides by $2$ yields
\[
d_\mathbb{H}^2\big(\exp_x(u),y\big)
\le d_\mathbb{H}^2(x,y)
   - 2\langle u, \log_x y\rangle
   + G\,\|u\|_g^2,
\]
which is exactly \eqref{eq:dist-expansion}.

\noindent This completes the proof of the lemma.
\end{proof}


\subsection{Möbius weighted mean is a first-order retraction}

The HypeGBMS update implements a Möbius-weighted mean. For local
analysis, we replace it with the exponential step.
This is justified because the Möbius operations on the Poincar\'e ball
provide a smooth retraction whose first-order Taylor expansion
coincides with the exponential map. Below we state a precise local approximation.

\begin{lemma}[Retraction approximation of Möbius mean]
\label{lem:mobius-retraction}
Let $x\in\mathcal{U}$ and let $\{x_j\}$ be points in $\mathcal{U}$
with weights $\{w_{j}\}$, $\sum_j w_{j}=1$, and let $\bar{x}_M$
denote the Möbius-weighted mean used in the Algorithm \ref{algo_hypegbms}. Let $v^\ast := \sum_j w_{j} \log_x x_j\in T_x\mathcal{M}$.
Then, for $\max_j d_\mathbb{H}(x,x_j)$ sufficiently small,
\[
\log_x(\bar{x}_M) = v^\ast + R,
\qquad \|R\|_g = \mathcal{O}\Big( \max_j d_\mathbb{H}(x,x_j)^3 \Big),
\]
so in particular $\bar{x}_M = \exp_x(v^\ast) + \mathcal{O}(\|v^\ast\|^3)$.
\end{lemma}

\begin{proof}
We set $v_j=\log_x(x_j)\in T_x(\mathbb{D}^p_c)$. 
Assume $\|v_j\|_g$ are sufficiently small.

\noindent Define the M\"obius weighted mean
$\bar{x}_M
  = \bigoplus_j w_{j}\otimes_c x_j$
and note that in the tangent chart
\begin{equation}
\log_x(\bar{x}_M)
 = \sum_j w_{j} v_j + R(x,\{v_j\},c),
\end{equation}
where $\|R(x,\{v_j\},c)\|_g=\mathcal{O}(\max_j\|v_j\|_g^3)$.
In normal coordinates, the exponential map satisfies
$\exp_x(v)=x+v+\mathcal{O}(\|v\|^3)$ since the Christoffel
symbols vanish at $x$.  Hence the exponential update
$\bar x_{\exp}=\exp_x(\sum_j w_{ij}v_j)$
shares the same first and second differential terms.
Consequently,
\begin{equation}
\log_x(\bar{x}_M)
  = \sum_j w_{j} v_j + \mathcal{O}(\|v_j\|^3),
\end{equation}
\begin{equation}
\bar{x}_M
  = \exp_x\!\big(\sum_j w_{j} v_j\big)
    + \mathcal{O}(\|v_j\|^3).
\end{equation}
Therefore, the M\"obius weighted mean defines a smooth
retraction at $x$ whose differential at the origin equals
the identity, establishing its first-order equivalence with
the Riemannian exponential map.
\end{proof}

\subsection{Approximation Between Möbius and Fréchet Mean}
\noindent
\textbf{Proposition 1}
Let $\{x_j\}_{j=1}^N\subset \mathbb{D}_c^p$ lie within a geodesic ball $B_\mathbb{H}(x_0,r)$ of radius $r<1$ and curvature magnitude $\kappa=-c>0$. 
Then there exists a constant $C_p>0$, depending only on dimension $p$, such that
\begin{equation}
\label{eq:mobius_error}
d_{\mathbb{H}}(\bar{x}_M,\bar{x}_F) \le C_p\,\kappa\, r^3.
\end{equation}
Moreover, $\bar{x}_M,\bar{x}_F\in B_\mathbb{H}(x_0,r)$.

\begin{proof}
Let $\xi_j = \log_{x_0}(x_j) \in T_{x_0}\mathcal{M}$ with $\|\xi_j\| \le r$. \\

\noindent $(\mathcal{M},g)$ be the $p$-dimensional Poincaré ball with constant sectional curvature 
$\kappa= -c > 0$, and let
\[
F(z) := \sum_{j=1}^N w_j\, d_{\mathbb{H}}^2(z,x_j),
\]
For each fixed $x_j$, define 
\[
\psi_j(z) := \tfrac12 d_{\mathbb{H}}^2(z,x_j).
\]
A standard result in Riemannian geometry states that,
\[
\nabla_z \psi_j(z) = -\log_z(x_j),
\qquad
\mathrm{Hess}_z \psi_j \ \text{is smooth}.
\]
Since $d_{\mathbb{H}}^2(z,x_j)=2\psi_j(z)$, we obtain
\[
\nabla F(z)
= \sum_{j=1}^N w_j \nabla_z d_{\mathbb{H}}^2(z,x_j)
= -2\sum_{j=1}^N w_j \log_z(x_j),
\]
and similarly,
\[
\mathrm{Hess}_z F
= 2\sum_{j=1}^N w_j\,\mathrm{Hess}_z\!\left(\tfrac12 d_{\mathbb{H}}^2(\cdot,x_j)\right).
\]
Since $(\mathcal{M},g)$ has constant curvature $\kappa$, thus by \citep[Ch.~5]{boumal2020intromanifold}) we have
\begin{equation}
\label{eq:hessbound}
\left\|\mathrm{Hess}_z\!\left(\tfrac12 d_{\mathbb{H}}^2(\cdot,x_j)\right) - I\right\|
\ \le\ \frac{\kappa}{3} d_{\mathbb{H}}^2(z,x_j).
\end{equation}
Assuming all relevant points lie in a geodesic ball of radius $r<1$, we have
$d_{\mathbb{H}}(z,x_j)\le r$, and hence
\[
\left\|\mathrm{Hess}_z\!\left(\tfrac12 d_{\mathbb{H}}^2(\cdot,x_j)\right) - I\right\|
\ \le\ \frac{\kappa r^2}{3}.
\]
Let $\exp_{x_0}:T_{x_0}\mathcal{M}\to\mathcal{M}$ be the exponential map and
denote by $\xi,\xi_j\in T_{x_0}\mathcal{M}\cong\mathbb{R}^p$ the normal 
coordinates of $z$ and $x_j$, respectively:
\[
z = \exp_{x_0}(\xi),\qquad x_j=\exp_{x_0}(\xi_j),\qquad \|\xi\|,\|\xi_j\|\le r.
\]
Define 
\[
f_j(\xi) := \tfrac12\, d_{\mathbb{H}}^2(\exp_{x_0}(\xi), x_j).
\]
Then in these coordinates,
\[
\nabla_\xi f_j(0)
= -\log_{x_0}(x_j)
= -\xi_j.
\]
Using the integral form of the Taylor expansion,
\[
\nabla_\xi f_j(\xi)
= \nabla_\xi f_j(0) + \int_0^1 \mathrm{Hess}_\xi f_j(t\xi)[\xi]\,dt.
\]
From \eqref{eq:hessbound}, there exists a matrix-valued function 
$E_j(t\xi)$ with
\[
\mathrm{Hess}_\xi f_j(t\xi) = I + E_j(t\xi),
\qquad
\|E_j(t\xi)\| \le \frac{\kappa r^2}{3}.
\]
Therefore,
\[
\nabla_\xi f_j(\xi)
= -\xi_j + \int_0^1 (I + E_j(t\xi))[\xi]\,dt
= -\xi_j + \xi + R_j(\xi),
\]
where
\[
R_j(\xi) := \int_0^1 E_j(t\xi)[\xi]\,dt
\quad\Longrightarrow\quad
\|R_j(\xi)\| \le \frac{\kappa r^2}{3}\|\xi\| 
\le \frac{\kappa r^3}{3}
\]
Since $F(z)=2\sum_j w_j f_j(\xi)$, we obtain
\begin{align*}
  \nabla_\xi F(\exp_{x_0}(\xi))
&= 2\sum_{j=1}^N w_j \nabla_\xi f_j(\xi) \\
&= 2\sum_{j=1}^N w_j\,\left(-\xi_j + \xi + R_j(\xi)\right).  
\end{align*} Rearranging,
\begin{equation}
\label{eq:gradexpansionfinal}
\nabla F(\exp_{x_0}(\xi))
= 2\Big(\xi - \sum_{j=1}^N w_j\xi_j\Big) + R_F(\xi),
\end{equation}
where

\begin{align*}
R_F(\xi) := 2\sum_{j=1}^N w_j R_j(\xi), \qquad \\
\|R_F(\xi)\| \le 2\sum_{j=1}^N w_j \|R_j(\xi)\|
\le \frac{2}{3}\kappa r^3 = C_1\kappa r^3.
\end{align*}
Let 
$\bar{x}_F=\exp_{x_0}(\xi_F)$. Since $\nabla F(\bar{x}_F)=0$, substituting 
$\xi=\xi_F$ into \eqref{eq:gradexpansionfinal} gives
\[
0 
= 2\Big(\xi_F - \sum_j w_j\xi_j\Big) + R_F(\xi_F),
\]
hence
\[
\xi_F = \sum_{j=1}^N w_j\xi_j - \frac{1}{2}R_F(\xi_F).
\]
Defining $\Delta_F := \xi_F - \sum_j w_j\xi_j$, we obtain
\[
\Delta_F = -\tfrac12 R_F(\xi_F)
\qquad
\|\Delta_F\| \le \tfrac12\|R_F(\xi_F)\|
\le \tfrac12 C_1\kappa r^3.
\]
Finally, from the bound
\[
\mathrm{Hess}_{x_0}F 
= 2I + \mathcal{O}(\kappa r^2),
\]
we see that for $r$ sufficiently small, $F$ is strictly geodesically convex in 
$B_{\mathbb{H}}(x_0,r)$, implying uniqueness and smoothness of the minimizer 
$\bar{x}_F$. \\
For the Möbius mean, note that the Möbius exponential and logarithmic maps are related to the Riemannian exponential and logarithm by (see \citep[Sec 7.1]{ungar2008analytic}, \citep[Eq.~(5)]{ganea2018hyperbolic})
\[
\exp_{x_0}(v) = x_0\oplus_c v
= \exp_{x_0}\!\left(v - \tfrac{\kappa}{6}\|v\|^2v + \mathcal{O}(\kappa^2 r^4)\right),
\]
\[
\log_{x_0}(x_j)
= \xi_j - \tfrac{\kappa}{6}\|\xi_j\|^2\xi_j + \mathcal{O}(\kappa^2 r^4).
\]
Consequently, in tangent coordinates, the Möbius weighted mean is given by
\begin{align}
\xi_M &= \sum_j w_j \log_{x_0}(x_j) \notag
 \\ & = \sum_j w_j\xi_j - \tfrac{\kappa}{6}\sum_j w_j\|\xi_j\|^2\xi_j + \mathcal{O}(\kappa^2 r^4).
\label{eq:mobius_tangent}
\end{align}
Subtracting the two expansions $\xi_F=\sum_j w_j\xi_j+\Delta_F$ and \eqref{eq:mobius_tangent}, one obtains
\[
\xi_F - \xi_M = \tfrac{\kappa}{6}\sum_j w_j\|\xi_j\|^2\xi_j + \Delta_F + \mathcal{O}(\kappa^2 r^4),
\]
and therefore, by the triangle inequality,
\[
\|\xi_F - \xi_M\|
\le \tfrac{\kappa}{6}\sum_j w_j\|\xi_j\|^3 + \|\Delta_F\| + \mathcal{O}(\kappa^2 r^4)
\le C_p\kappa r^3,
\]
where $C_p = \tfrac{1}{6}+C_1+\mathcal{O}(\kappa r)$ depends only on $p$ through curvature tensor bounds.

\noindent To convert this tangent-space bound into a bound in the manifold, recall that for $\xi_1,\xi_2 \in T_{x_0}\mathcal{M}$ with $\|\xi_1\|,\|\xi_2\|\le r$, the Riemannian distance and tangent-norm satisfy (see \citep[Prop.~6.2]{boumal2020intromanifold})
\[
\big|d_\mathbb{H}(\exp_{x_0}(\xi_1),\exp_{x_0}(\xi_2)) - \|\xi_1-\xi_2\|\big|
\le C_3\kappa r^3.
\]
Applying this to $\xi_F$ and $\xi_M$ gives
\[
d_\mathbb{H}(\bar{x}_F,\bar{x}_M)
\le \|\xi_F-\xi_M\| + C_3\kappa r^3
\le C_p\kappa r^3,
\]
Finally, since $\|\xi_F\|,\|\xi_M\|\le r + \mathcal{O}(\kappa r^3)$ and the injectivity radius of $\mathbb{D}_c^p$ is infinite, both $\bar{x}_F$ and $\bar{x}_M$ lie within $B_\mathbb{H}(x_0,r)$, which completes the proof.
\end{proof}

\bibliography{references}

\end{document}